\documentclass[sigconf,nonacm]{acmart}

\usepackage[utf8]{inputenc}
\usepackage[T1]{fontenc}
\usepackage{booktabs}
\usepackage{graphicx}
\usepackage{xcolor}
\usepackage{geometry}
\usepackage{enumitem}
\usepackage{tikz}
\usetikzlibrary{shapes.geometric, arrows.meta, positioning, fit, backgrounds, calc}
\usepackage{array}
\usepackage{hyperref}

\definecolor{accent}{HTML}{2563EB}
\definecolor{accent2}{HTML}{7C3AED}
\definecolor{accent3}{HTML}{059669}
\definecolor{lightbg}{HTML}{F0F5FF}
\definecolor{lightbg2}{HTML}{F5F0FF}
\definecolor{lightbg3}{HTML}{ECFDF5}
\definecolor{graytext}{HTML}{4B5563}

\hypersetup{
    colorlinks=true,
    linkcolor=accent,
    citecolor=accent2,
    urlcolor=accent
}

\usepackage{algorithm}
\usepackage{algpseudocode}
\usepackage{multirow}
\usepackage{xspace}

\newcommand{\method}{HyBrain\xspace}

\usepackage{balance}

\title{Spatiotemporal Hyperedges for \\EEG Seizure Detection and Prediction}
\author{Hyunju Kim}
\authornote{Both authors contributed equally to this research.}
\email{hyunju@gatech.edu}
\affiliation{%
  \department{School of Computational Science and Engineering}
  \institution{Georgia Institute of Technology}
  \city{Atlanta}
  \state{GA}
  \country{USA}
}

\author{Sheo Yon Jhin}
\authornotemark[1]
\email{sheoyon.jhin@kaist.ac.kr}
\affiliation{%
  \department{School of Computing}
  \institution{Korea Advanced Institute of Science and Technology}
  \city{Daejeon}
  \country{Republic of Korea}
}

\author{Noseong Park}
\email{noseong@kaist.ac.kr}
\affiliation{%
  \department{School of Computing}
  \institution{Korea Advanced Institute of Science and Technology}
  \city{Daejeon}
  \country{Republic of Korea}
}

\author{Nabil Imam}
\email{nimam6@gatech.edu}
\affiliation{%
  \department{School of Computational Science and Engineering}
  \institution{Georgia Institute of Technology}
  \city{Atlanta}
  \state{GA}
  \country{USA}
}

\renewcommand{\shortauthors}{Hyunju Kim, Sheo Yon Jhin, Noseong Park, \& Nabil Imam}

\begin{document}
\pagestyle{plain}

\begin{abstract}
Seizure detection and prediction from EEG are clinically important but challenging because seizures are rare, temporally localized, and propagate as coordinated events across multiple channels. Recent dynamic graph neural networks model this by running a temporal model over a sequence of per-time-step pairwise channel edges. However, this pairwise construction misses the spatiotemporal coupling that constitutes a seizure, at substantial training cost.
We propose \method, which summarizes spatiotemporal EEG evidence through a small set of soft hyperedges rather than pairwise edges. A per-channel Mamba backbone produces one token per (channel, second), and a spatiotemporal hyperedge block pools these tokens into $E_h$ shared group embeddings through soft memberships and broadcasts them back. The same encoder serves three downstream tasks: window-based detection, one-second point-wise detection, and preictal seizure prediction.
On TUSZ and CHB-MIT, \method achieves the best AUROC on every reported setting against ten baselines, with the largest gap on long-clip preictal prediction. It also matches the most efficient baselines in training time and peak GPU memory. A qualitative analysis shows that even a single learned hyperedge cleanly captures the preictal $\to$ ictal $\to$ postictal trajectory on a real seizure clip.
\end{abstract}

\maketitle
\thispagestyle{plain}

\section{Introduction}
Epilepsy affects approximately 50 million people worldwide, and timely recognition of seizure activity is critical for preventing injuries and life-threatening complications~\cite{devinsky2016sudden, engel2013seizures,mormann2007seizure,kuhlmann2018seizure}.
In clinical practice, seizure monitoring is performed on continuous multi-channel electroencephalography (EEG), a long and high-dimensional time series in which seizures occupy only brief intervals scattered across hours to days of recording~\cite{bergey2015long}.
Reliably surfacing these rare events from large-scale streaming biosignal data is impractical for manual review, making automated EEG seizure detection a long-standing problem at the intersection of biomedical informatics and time-series data mining~\cite{shah2018temple, shoeb2009application}.

\begin{figure}
\vspace{2em}
  \centering
  \includegraphics[width=0.98\linewidth]{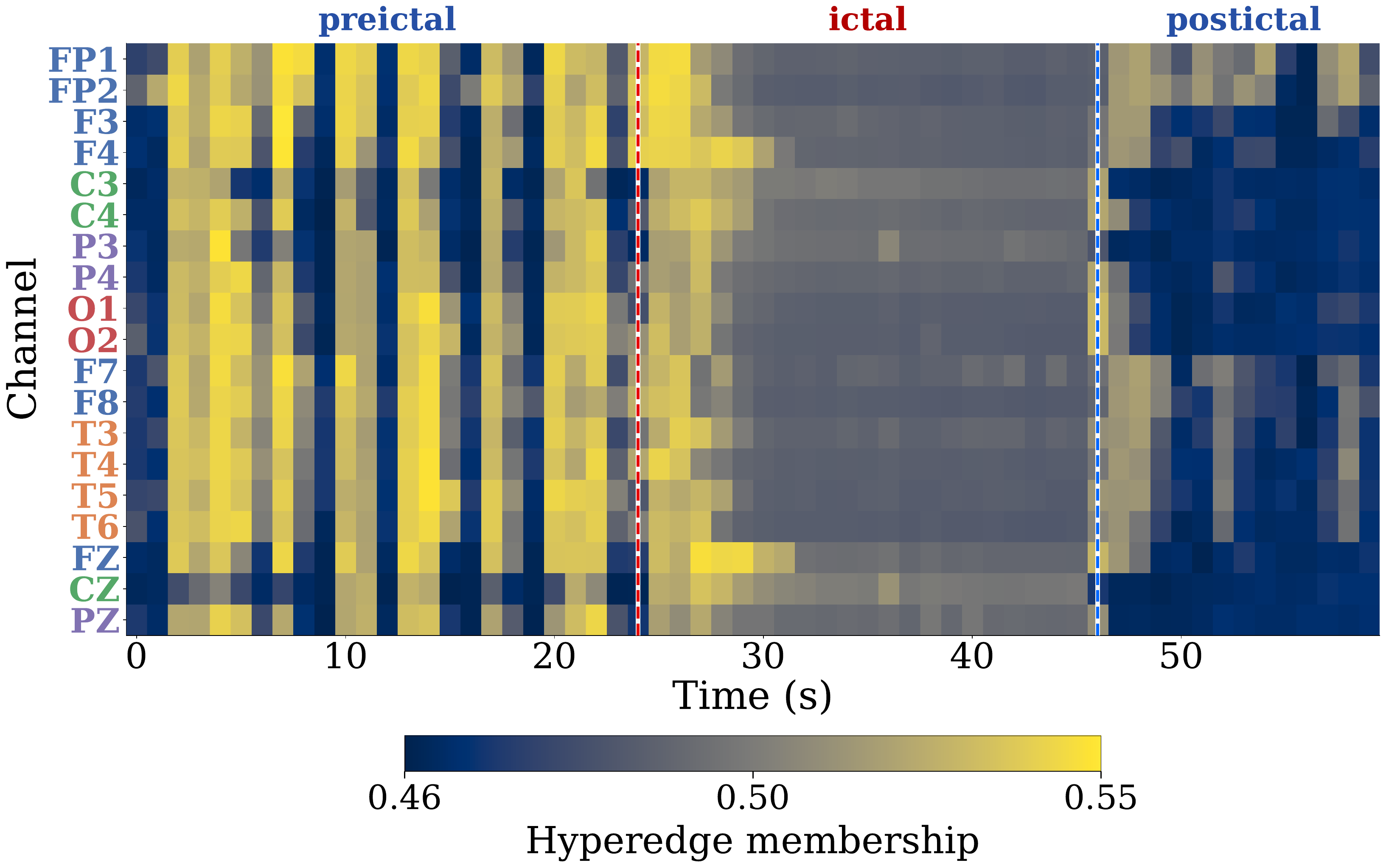}
  \caption{Membership of one spatiotemporal hyperedge over the channel-time tokens of a 60s TUSZ clip. The membership regime shifts across phases (dashed lines), coupling the tokens of each phase into a distinguishable group-level embedding. This emerges because our model learns hyperedges well-suited for seizure detection.
  }
  \Description{Heatmap of one learned hyperedge's membership over the channel-time tokens of a 60-second TUSZ clip; the membership pattern is bursty in the preictal interval, flat near the mean during the annotated seizure, and uniformly low in the postictal interval.}
  \vspace{-2em}
  \label{fig:teaser}
\end{figure}

Three families of methods have shaped EEG seizure modeling. Early models treat EEG as a multi-channel time series and apply CNNs or RNNs along the time axis only, leaving the relations among electrodes implicit~\cite{lawhern2018eegnet, ahmedt2020neural}. A second line constructs a static pairwise graph over channels and runs a GNN on top, capturing electrode relations but not how they change during a seizure~\cite{covert2019temporal, tang2021self}. To capture time-varying connectivity, recent methods introduce dynamic graph neural networks (GNNs) that replace the static adjacency with a pairwise channel graph at each time step~\cite{li2018dcrnn_traffic, pareja2020evolvegcn, tang2023modeling}. These models see both spatial and temporal variation, but each relation is still defined between two channels within the same time step. A heavier subfamily further applies a temporal model to the per-time-step edge sequence~\cite{gao2022equivalence, kotoge2026evobrain}, explicitly tracking how each pairwise edge evolves over time. However, this still does not directly model cross-time, cross-channel coupling, such as one channel at one time influencing another channel at a later time. It also decomposes a multi-channel seizure event into many bilateral edge trajectories and requires retaining edge-level activations over time,  which increases the computational and memory cost of the graph pipeline.

We address these limitations with \method, which learns a small set of spatiotemporal hyperedges over channel-time tokens. This formulation models channel-time coupling and group-level interaction directly, and does not require dense per-time-step pairwise graphs.
The encoder follows a node $\to$ hyperedge $\to$ node flow.
A per-channel Mamba backbone~\cite{gu2023mamba} first produces one token per (channel, second) pair.
A spatiotemporal hyperedge block then pools all $NT$ tokens into $E_h$ shared group embeddings through soft memberships, and broadcasts the same $E_h$ group embeddings back to every token weighted by its own memberships, so each token absorbs the group-level context relevant to it (cf. Figure~\ref{fig:teaser}).
A Set-Transformer--style Pooling by Multi-head Attention (PMA) readout~\cite{lee2019set} finally aggregates the updated tokens for each task head.
This design summarizes the entire channel-time field through $E_h\!\ll\!NT$ group embeddings rather than reconstructing it from per-time-step pairwise edges.

We evaluate \method on TUSZ~\cite{shah2018temple} and CHB-MIT~\cite{shoeb2009application} across three tasks: window-based detection, one-second point-wise detection, and preictal seizure prediction.
The same encoder supports all three tasks because \method learns task-relevant hyperedge memberships over channel-time tokens, while only the lightweight readout head is changed for each task (cf. Figure~\ref{fig:teaser}).

Our contributions are:
\begin{enumerate}
  \item We introduce learnable hyperedges over channel-time tokens that model cross-time and cross-channel coupling without pairwise decomposition.
  \item We use the same encoder for window-based detection, point-wise detection, and preictal prediction by learning task-relevant hyperedge memberships, with only the readout head changed per objective.
\item Against ten baselines, \method achieves the best AUROC across all reported detection and prediction settings on TUSZ and CHB-MIT, with the largest gain on long-clip preictal prediction, while matching the most efficient baselines in time and memory.
\end{enumerate}
Code is available at \url{https://github.com/hhyy0401/seizure}.

\label{sec:intro}

\section{Related Work}

\begin{figure*}
  \centering
  \includegraphics[width=0.96\linewidth]{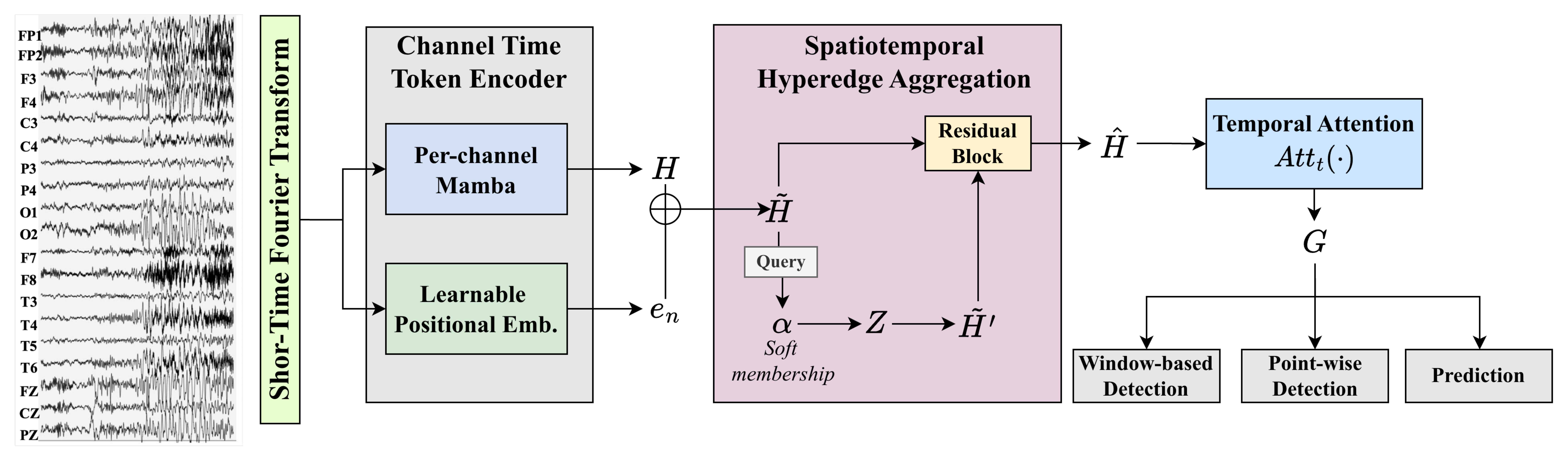}
\caption{Overview of \method. A multi-channel EEG clip is first converted into one-second channel-time spectral tokens using a per-channel short-time Fourier transform (STFT). The channel-time token encoder then processes each channel independently with a Mamba backbone and adds a learnable channel embedding to preserve channel identity. The spatiotemporal hyperedge aggregation block pools all channel-time tokens into \(E_h\) learnable hyperedge embeddings through soft memberships and broadcasts the resulting group-level context back to the tokens, capturing high-order cross-channel and cross-time interactions. A temporal attention layer further refines intra-channel temporal dependencies, producing the shared representation \(G\). Finally, task-specific heads adapt this shared representation to window-based seizure detection, point-wise seizure detection, and preictal seizure prediction.}
\Description{Block diagram of the \method encoder showing per-channel STFT features, a Mamba-based channel-time token encoder with learnable channel embeddings, spatiotemporal hyperedge aggregation with soft memberships and hyperedge embeddings, temporal attention refinement, and task-specific heads for window-based detection, point-wise detection, and seizure prediction.}
  \label{fig:overall_workflow}
\end{figure*}

\subsection{EEG Deep Learning Models}

\paragraph{Sequential and static-graph models.}
Sequential models treat EEG as a multi-channel time series and apply CNNs or RNNs mainly along the temporal axis, with EEGNet~\cite{lawhern2018eegnet} as a general-purpose backbone and CNN--LSTM~\cite{ahmedt2020neural} as an early seizure-focused variant. Static-graph approaches for EEG seizure detection make electrode topology explicit through a fixed adjacency: TGCN~\cite{covert2019temporal} introduced this idea with a hand-designed spatial graph, while self-supervised graph neural networks~\cite{tang2021self} added contrastive pretraining to alleviate label scarcity. A recent line instead pretrains large EEG foundation models on broad corpora and finetunes them on downstream clinical EEG tasks, including BIOT~\cite{yang2023biot}, LaBraM~\cite{jiang2024large}, and EEGPT~\cite{wang2024eegpt}.

\paragraph{Dynamic GNNs.}
Because a seizure unfolds as a time-varying pattern of inter-channel coordination, a dedicated line of EEG seizure detection dynamically updates the channel graph within each EEG clip. DCRNN~\cite{li2018dcrnn_traffic} pairs a diffusion-convolutional RNN with a per-clip $k$-NN correlation graph, EvolveGCN~\cite{pareja2020evolvegcn} evolves GCN parameters across time steps with a recurrent network, and GraphS4mer~\cite{tang2023modeling} relearns the graph at each within-clip segment and combines it with a structured state-space backbone~\cite{gu2022efficiently, gu2023mamba}. Recent state-of-the-art models further track temporal dynamics over edge sequences, for example with a per-edge GRU in GRU-GCN~\cite{gao2022equivalence} or a per-edge Mamba in EvoBrain~\cite{kotoge2026evobrain}. However, these methods still use pairwise edges as the basic relational unit, so higher-order coordination among multiple channels and its temporal consistency are only modeled indirectly through sequential updates or stacked layers.

\paragraph{Hypergraph learning.}
Hypergraphs extend ordinary graphs by connecting multiple nodes with one hyperedge, making them suitable for high-order group interactions. Prior hypergraph neural networks, including HGNN~\cite{feng2019hypergraph}, HyperGCN~\cite{yadati2019hypergcn}, and AllSet~\cite{chien2022you}, usually use fixed incidences from predefined groups or $k$-NN features. In EEG, STHGCN~\cite{li2023feature} applies this idea to emotion recognition by constructing separate spatial and temporal hypergraphs over EEG features. \method instead learns input-conditioned hyperedges jointly over the channel-time grid, using a small fixed hyperedge budget for each clip.

\subsection{EEG Seizure Tasks}

\paragraph{Point-wise seizure detection.}

Clinical EEG datasets such as TUSZ~\cite{shah2018temple} annotate seizure onset and offset at one-second resolution, and detectors are expected to localize these transitions, not just classify entire clips. Classical seizure-detection systems reflected this directly, combining point-wise scoring at one-second resolution with event-level scoring based on onset latency and false-alarm rate per hour~\cite{wilson2003seizure, ahammad2014detection, shoeb2009application}, and early deep-learning detectors preserved this resolution through CNN--RNN or temporal GCN backbones with per-second outputs~\cite{craley2022automated, saab2020weak}. The recently proposed SzCORE benchmark~\cite{dan2024szcore} likewise lists both point-wise and event-level scoring among its recommended evaluation metrics. Many recent dynamic-GNN detectors~\cite{tang2023modeling, kotoge2026evobrain}, by contrast, adopt coarser supervision and assign a single binary label per 12\,s or 60\,s window, leaving the available second-level annotations unused at training time. We complement this line by re-introducing point-wise detection as a first-class training objective, and report per-second results alongside the window-level ones.

\paragraph{Seizure prediction.}
Seizure prediction differs from detection in that the discriminative evidence is diffuse and distributed across channels, evolving slowly over the preictal window rather than concentrated in a clearly ictal interval. Unlike point-wise detection, prediction is usually formulated as distinguishing preictal from interictal EEG under a specified seizure prediction horizon and seizure occurrence period~\cite{winterhalder2003seizure,chen2020performance}. Early prediction studies and clinical reviews emphasized patient-specific preictal dynamics, early-warning time, and false-alarm burden as central challenges for deployable systems~\cite{mormann2007seizure,gadhoumi2016seizure,kuhlmann2018seizure}. Recent deep-learning work has further revisited this setting as an early-warning task rather than only a detection problem~\cite{truong2018convolutional,kotoge2026evobrain}. Prediction therefore provides a complementary evaluation setting to detection. Rather than recognizing ictal patterns that are already expressed, the model must infer a latent preictal state from subtle changes in ongoing background activity. This setting tests whether \method can aggregate weak evidence over long clips through learned channel--time hyperedges.

\label{sec:related}

\section{Proposed Methods}

\subsection{Problem Formulation}
\label{sec:problem}

We represent each EEG clip as a tensor $X \in \mathbb{R}^{N \times T \times F}$, where $N$ is the number of EEG channels, $T$ is the number of one-second time steps, and $F$ is the dimension of the per-second short-time Fourier transform (STFT) feature.
A clip of length $L$ seconds has $T=L$ time steps, and $x_{n,t}\in\mathbb{R}^{F}$ denotes the spectral feature of channel $n$ at second $t$.

Following standard epileptological convention, we refer to time intervals as \emph{ictal} during an annotated seizure,  \emph{preictal} immediately preceding seizure onset, \emph{postictal} immediately following seizure offset, and \emph{interictal} otherwise.

On this shared input, we study three binary tasks.
\begin{enumerate}
\item \textbf{Window-based detection} predicts $y^{\mathrm{win}}\in\{0,1\}$, indicating whether the clip overlaps with seizure interval.
\item \textbf{Point-wise detection} predicts $\mathbf{y}^{\mathrm{pt}}\in\{0,1\}^{T}$, where $y^{\mathrm{pt}}_t=1$ if the $t$-th second overlaps with a seizure.
\item \textbf{Seizure prediction} predicts $y^{\mathrm{pre}}\in\{0,1\}$, the binary classification between interictal (normal) and preictal states, where the preictal state is defined as the $\Delta_p{=}60$\,s window immediately before each seizure onset~\cite{kotoge2026evobrain}.
\end{enumerate}

All tasks use the same encoder to map $X$ into a channel-time representation $G\in\mathbb{R}^{N\times T\times d}$, and task-specific heads produce either a clip-level logit or per-second logits.

\subsection{Overall Workflow}
\label{sec:overall}

Figure~\ref{fig:overall_workflow} shows the detailed design of our method, \method. The overall workflow is as follows:
\begin{enumerate}
    \item A per-channel STFT converts each channel into a sequence of one-second spectral features.
    \item In the \textbf{Channel-time Token Encoder}, a per-channel Mamba backbone~\cite{gu2023mamba} encodes these features into one token per (channel, second) pair, augmented with a learnable channel embedding.
    \item In the \textbf{Spatiotemporal Hyperedge Aggregation} block, a small set of $E_h$ learnable hyperedges pools the channel-time tokens into group embeddings through soft memberships and broadcasts them back to every token, capturing high-order interactions across channels and time without explicit dynamic graph construction.
    \item A \textbf{Temporal Attention} layer then refines long-range temporal dependencies, producing the shared representation $G \in \mathbb{R}^{N \times T \times d}$.
    \item \textbf{Task-Specific Heads} based on PMA-style channel readouts~\cite{lee2019set} adapt $G$ to window-based detection, point-wise detection, and seizure prediction.
\end{enumerate}

\subsection{Channel-Time Token Encoding}
\label{sec:temporal-feature}

Each EEG channel is first converted into a one-second spectral sequence by a per-channel STFT, yielding $x_{n,t}\in\mathbb{R}^{F}$ per (channel, second) pair (preprocessing details in Section \ref{sec:expsetting}). We then encode temporal dynamics within each EEG channel. For each channel $n$, the spectral sequence $x_{n,1:T}$ is processed by a stack of $L_M$ Mamba layers~\cite{gu2023mamba}, each a selective state-space module that mixes information along the temporal axis only. Let $u^{(0)}_{n,t} = W_{\mathrm{in}} x_{n,t} \in \mathbb{R}^{d}$, where $W_{\mathrm{in}}\in\mathbb{R}^{d\times F}$ is a learnable linear map that lifts the spectral features to the hidden dimension $d$, and $\mathrm{LN}(\cdot)$ denotes layer normalization. At layer $l$, a latent state $s^{(l)}_{n,t}$ evolves under a selective recurrence and the token is updated by a residual layer-norm:
\begin{align}
s^{(l)}_{n,t} &= \overline{A}_{n,t} s^{(l)}_{n,t-1} + \overline{B}_{n,t} u^{(l-1)}_{n,t}, \\
u^{(l)}_{n,t} &= \mathrm{LN}\left(u^{(l-1)}_{n,t} + C_{n,t} s^{(l)}_{n,t}\right),
\end{align}
where the coefficients $(\overline{A}_{n,t}, \overline{B}_{n,t}, C_{n,t})$ depend on $u^{(l-1)}_{n,t}$ and follow Mamba's HiPPO-initialized form~\cite{gu2023mamba}. Setting $H_{n,t} = u^{(L_M)}_{n,t}$ yields the backbone output $H\in\mathbb{R}^{N\times T\times d}$. Because the recurrence is per-channel, the $N$ channels are processed independently and no cross-channel mixing happens at this stage. To preserve channel identity before hyperedge aggregation, we add a learnable channel embedding $e_n\in\mathbb{R}^{d}$ as a channel positional embedding:
\begin{equation}
\tilde{H}_{n,t} = H_{n,t} + e_n.
\end{equation}

\subsection{Spatiotemporal Hyperedge Aggregation}
\label{sec:hyperedge}

Seizure events often appear as channel-to-channel coupling that may carry a temporal lag. We therefore introduce a spatiotemporal hyperedge block that proceeds in three steps: it first learns soft token-to-hyperedge memberships, then aggregates the tokens into hyperedge embeddings, and finally updates each token by broadcasting the incident hyperedges back.

We flatten $\tilde{H}\in\mathbb{R}^{N\times T\times d}$ into $NT$ channel-time tokens $\{\tilde{h}_i\}_{i=1}^{NT}$, where each $\tilde{h}_i\in\mathbb{R}^{d}$ corresponds to one channel at one time step, and introduce $E_h$ learnable hyperedge queries $\{q_k\}_{k=1}^{E_h}\subset\mathbb{R}^{d}$, where $E_h$ is the number of latent hyperedges and each $q_k$ represents one latent spatiotemporal hyperedge. For token $i$ and hyperedge $k$, the soft membership is
\begin{equation}
\alpha_{i,k} = \sigma\left(\frac{q_k^\top \tilde{h}_i}{\sqrt{d}}\right),
\label{eq:member}
\end{equation}
where $\sigma(\cdot)$ is the sigmoid activation. We use sigmoid rather than softmax across hyperedges so that a token can belong to multiple hyperedges, supporting overlapping spatiotemporal patterns.

Given the memberships, each hyperedge embedding $z_k$ is the membership-weighted average of its assigned tokens, and each token is updated by broadcasting the incident hyperedge embeddings back with the same weights:
\begin{equation}
z_k = \frac{\sum_{i=1}^{NT} \alpha_{i,k} \tilde{h}_i}{\sum_{i=1}^{NT} \alpha_{i,k}},
\quad
\tilde{h}'_i = \sum_{k=1}^{E_h} \alpha_{i,k} z_k.
\label{eq:node}
\end{equation}
Here $z_k$ forms a group-level representation of one coordinated spatiotemporal pattern, and $\tilde{h}'_i$ encodes the group-level spatiotemporal context associated with token $i$.

We then fuse the original and updated tokens through a post-norm residual block:
\begin{equation}
\hat{h}_i = \mathrm{LN}\left(\tilde{h}_i + \mathrm{GELU}\left(W_{\mathrm{out}} \tilde{h}'_i\right)\right),
\label{eq:residual}
\end{equation}
where $W_{\mathrm{out}}\in\mathbb{R}^{d\times d}$ is a learnable linear map and $\mathrm{GELU}(\cdot)$ is the activation. Intuitively, $\tilde{h}_i$ provides the \emph{local} channel-time evidence at position $i$, while $\mathrm{GELU}(W_{\mathrm{out}} \tilde{h}'_i)$ injects the \emph{global}, hyperedge-aggregated spatiotemporal context, so that the post-norm residual incorporates the two views.

We stack $L_h$ such blocks and denote the resulting tensor by $\hat{H}^{(L_h)}\in\mathbb{R}^{N\times T\times d}$. We then add a temporal attention layer to capture residual long-range dependencies within each channel.
\begin{equation}
G_{n,:} = \mathrm{LN}\left(\hat{H}^{(L_h)}_{n,:} + \beta\,\mathrm{Att}_t\left(\hat{H}^{(L_h)}_{n,:}\right)\right),
\end{equation}
Here, $\mathrm{Att}_t(\cdot)$ denotes multi-head self-attention along the time axis with weights shared across channels, and $\beta$ is a gate coefficient. The output $G\in\mathbb{R}^{N\times T\times d}$ is used as the shared representation for downstream readout modules.

\subsection{Task-Specific Heads and Training Objective}
\label{sec:heads}

The shared representation $G$ is passed to task-specific heads. All heads use a PMA-style channel readout~\cite{lee2019set} followed by a linear classifier, where the difference is whether the temporal axis is pooled before the readout or preserved for per-second prediction.

For \textbf{window-based detection} and \textbf{seizure prediction}, we use the same clip-level readout:
\begin{equation}
r = \mathrm{PMA}\left(\mathrm{Pool}_t(G)\right), \quad
\hat{y} = f(r).
\end{equation}
Here, $\mathrm{Pool}_t(\cdot)$ denotes average pooling over time, $f(\cdot)$ is the task-specific linear classifier, and $\hat{y}$ is a raw logit.
The two tasks use different labels, $y^{\mathrm{win}}$ and $y^{\mathrm{pre}}$, and are trained separately with binary cross-entropy on the sigmoid of the logit.

For \textbf{point-wise detection}, we keep the temporal axis and apply the PMA channel readout independently at each time step:
\begin{equation}
r_t = \mathrm{PMA}\left(G_{:,t}\right), \quad
\hat{y}_t = f_{\mathrm{pt}}(r_t), \quad t=1,\dots,T,
\end{equation}
with PMA parameters shared across all time steps.
To encode the temporal contiguity of seizure labels directly into the objective, we add a small adjacent-frame logit-smoothness regularizer:
\begin{equation}
\mathcal{L}_{\mathrm{pt}} = \frac{1}{T}\sum_{t=1}^{T} \mathrm{BCE}(\hat{y}_t, y^{\mathrm{pt}}_t) + \lambda \cdot \frac{1}{T-1} \sum_{t=1}^{T-1} \left(\hat{y}_{t+1} - \hat{y}_t\right)^2.
\end{equation}
Using logit-space differences preserves informative gradients under confident predictions, while penalizing only consecutive differences allows sharp label transitions at seizure onset and offset.

\subsection{Theoretical Analysis}
\label{sec:rankk}

We show that \method's hyperedge block induces an input-adaptive global mixing operator of rank at most $E_h$ over all channel-time tokens, and compare its compute and activation-memory cost against representative edge-stream dynamic-GNN baselines.

\paragraph{Setup.}
A dynamic pairwise GNN over $N$ EEG channels and $T$ time steps maintains an adjacency $A_t \in \mathbb{R}^{N \times N}$ at every time step, with up to $E \le N^2$ pairwise edges per time step. Recent edge-stream variants further apply a temporal model along the edge sequence: GRU-GCN~\cite{gao2022equivalence} uses a per-edge GRU, and EvoBrain~\cite{kotoge2026evobrain} uses a per-edge Mamba. \method instead applies a single hyperedge block with $E_h$ shared hyperedges to the flattened $NT$ channel-time tokens.

\begin{proposition}[Rank-$E_h$ global mixing of channel-time tokens]
\label{prop:rank-k}
Fix the soft memberships $\alpha \in \mathbb{R}^{NT \times E_h}$ from
Eq.~(\ref{eq:member}), and let $s_k = \sum_{i=1}^{NT} \alpha_{i,k}$.
The aggregate-and-broadcast step in Eq.~(\ref{eq:node}) acts on the
channel-time tokens as
\[
\tilde{h}'_i = \sum_{j=1}^{NT} M_{ij}\, \tilde{h}_j,
\qquad
M_{ij} = \sum_{k=1}^{E_h} \frac{\alpha_{i,k}\, \alpha_{j,k}}{s_k}.
\]
The induced operator $M = A D^{-1} A^\top$, with $A = [\alpha_{i,k}]
\in \mathbb{R}^{NT \times E_h}$ and $D = \mathrm{diag}(s_1, \ldots, s_{E_h})$,
satisfies two properties:
\begin{enumerate}
    \item \textbf{Global support.} Under sigmoid memberships,
    $\alpha_{i,k} > 0$ for all $i, k$, so $M_{ij} > 0$ for every pair
    of channel-time tokens.
    \item \textbf{Low-rank global interaction.}
    $M$ factors through the $E_h$
    hyperedges via $A$, so $\mathrm{rank}(M) \le E_h$.
\end{enumerate}
\end{proposition}

\begin{proof}
Substituting the hyperedge embedding into the broadcast update gives
\[
\tilde h'_i
=
\sum_{k=1}^{E_h}\alpha_{i,k}
\sum_{j=1}^{NT}\frac{\alpha_{j,k}}{s_k}\tilde h_j
=
\sum_{j=1}^{NT}
\left(
\sum_{k=1}^{E_h}\frac{\alpha_{i,k}\alpha_{j,k}}{s_k}
\right)\tilde h_j .
\]
Thus $M=AD^{-1}A^\top$, where $A=[\alpha_{i,k}]\in\mathbb{R}^{NT\times E_h}$ and $D=\mathrm{diag}(s_1,\ldots,s_{E_h})$. Therefore $\mathrm{rank}(M)\le E_h$. Since sigmoid memberships are strictly positive for finite inputs, $M_{ij}>0$ for all $i,j$. 
\end{proof}

Proposition~\ref{prop:rank-k} shows that the hyperedge block couples every pair of channel-time tokens in one step, while routing the interaction through only $E_h$ shared hyperedges. Table~\ref{tab:cost-analysis} compares the resulting per-sample asymptotic cost against the two edge-stream variants used by the strongest baselines. Compared with a Mamba-style edge stream, the edge-stream-to-hyperedge compute ratio scales as $\mathcal{O}(E/(NE_h))$, so the hyperedge block is cheaper whenever the number of pairwise edges per time step exceeds the effective hyperedge budget $NE_h$. For dense pairwise graphs with $E=\mathcal{O}(N^2)$, this yields a compute improvement of $\mathcal{O}(N/E_h)$ and an edge-related activation-memory reduction of approximately $\mathcal{O}(Nd/E_h)$ when the membership tensor $TNE_h$ dominates the $E_hd$ hyperedge state. For sparse graphs with $E=cN$, where $c$ is the average number of pairwise edges per node, the edge-stream-to-hyperedge compute ratio becomes $c/E_h$, and the corresponding relation-related activation-memory ratio is approximately $cd/E_h$ under the same dominance assumption.
\begin{table}[t]
\centering
\caption{Per-sample asymptotic cost of representative spatiotemporal interaction blocks.}
\label{tab:cost-analysis}
\small
\begin{tabular}{lcc}
\toprule
Method & Compute & Activation memory \\
\midrule
GRU-style edge stream~\cite{gao2022equivalence} & $\mathcal{O}(TEd^2)$ & $\mathcal{O}(TEd)$ \\
Mamba-style edge stream~\cite{kotoge2026evobrain} & $\mathcal{O}(TEd)$ & $\mathcal{O}(TEd)$ \\
\method hyperedge block & $\mathcal{O}(TNE_hd)$ & $\mathcal{O}(TNE_h + E_hd)$ \\
\bottomrule
\end{tabular}
\end{table}

Empirically, these claims are checked in three places. The effectiveness of the hyperedge coupling itself is tested in Section~\ref{sec:ablation_block}. The memory and runtime advantage is evaluated in Section~\ref{sec:effective}. The sufficiency of a small $E_h$ is verified by the hyperedge sensitivity study in Section~\ref{sec:eh_sensitivity}.

\label{sec:methods}

\section{Experiments}
\subsection{Experimental Settings}
\label{sec:expsetting}

\paragraph{Datasets and preprocessing.}
We evaluate on TUSZ~\cite{shah2018temple} and CHB-MIT~\cite{shoeb2009application}, using 12\,s and 60\,s clip settings for window-based experiments and following the data splits and evaluation protocol of EvoBrain~\cite{kotoge2026evobrain}.
For the input feature, raw EEG is resampled to 200\,Hz, segmented into non-overlapping 1-second windows per channel, and each window is mapped to its log-magnitude STFT over the $F{=}100$ positive frequency bins.
An $L$-second clip therefore yields $X\in\mathbb{R}^{N\times L\times 100}$.
For each subject, per-channel STFT mean and standard deviation are computed on the training split and used to normalize all splits. 

For point-wise detection, we derive a per-second label $y^{\text{pt}}\in\{0,1\}^{L}$ for each clip, where $y^{\text{pt}}_t{=}1$ iff the $t$-th 1-second window overlaps any annotated seizure interval. The window-level label is recovered as $y^{\text{win}}=\max_t y^{\text{pt}}_t$.

\begin{table*}[t]
\centering
\caption{Window-based seizure detection performance on TUSZ and CHB-MIT. Mean $\pm$ std across three seeds. The best result is highlighted in \textbf{bold}, and the second-best is \underline{underlined}. Time denotes wall-clock seconds per training epoch on a single A6000, and Memory denotes peak GPU memory during training, reported in MB.}
\vspace{-1em}
\label{tab:tusz_main}
\small
\setlength{\tabcolsep}{4pt}
\renewcommand{\arraystretch}{1.05}
\begin{tabular}{lcccccccc}
\toprule
\multirow{2}{*}{Method} & \multirow{2}{*}{Time (s/ep.)} & \multirow{2}{*}{Memory (MB)}
 & \multicolumn{2}{c}{TUSZ (12\,s)} & \multicolumn{2}{c}{TUSZ (60\,s)}
 & \multicolumn{2}{c}{CHB-MIT (12\,s)} \\
\cmidrule(lr){4-5} \cmidrule(lr){6-7} \cmidrule(lr){8-9} 
& & & AUROC & F1 & AUROC & F1 & AUROC & F1 \\
\midrule
LSTM        & 31.30 & 88.0 & 0.839 $\pm$ 0.016 & 0.403 $\pm$ 0.065 & 0.826 $\pm$ 0.050 & 0.427 $\pm$ 0.067 & 0.823 $\pm$ 0.075 & 0.064 $\pm$ 0.007 \\
CNN-LSTM    & 29.90 & 1208.1& 0.824 $\pm$ 0.014 & 0.365 $\pm$ 0.053 & 0.690 $\pm$ 0.026 & 0.242 $\pm$ 0.020 & 0.798 $\pm$ 0.110 & 0.049 $\pm$ 0.032 \\
BIOT        & 86.07 & 1850.7 & 0.811 $\pm$ 0.006 & 0.301 $\pm$ 0.013 & 0.717 $\pm$ 0.060 & 0.292 $\pm$ 0.036 & 0.903 $\pm$ 0.018 & 0.054 $\pm$ 0.090 \\
LaBraM      & 212.9 & 3576.7 & 0.863 $\pm$ 0.010 & 0.398 $\pm$ 0.034 & 0.865 $\pm$ 0.007 & 0.466 $\pm$ 0.016 & 0.787 $\pm$ 0.010 & 0.054 $\pm$ 0.022 \\
EEGPT       & 151.2 & 1022.1 & 0.884 $\pm$ 0.004 & 0.438 $\pm$ 0.010 & 0.784 $\pm$ 0.009 & 0.343 $\pm$ 0.007 & 0.911 $\pm$ 0.016 & \textbf{0.266 $\pm$ 0.025} \\
EvolveGCN   & 72.89 & \textbf{62.3} & 0.812 $\pm$ 0.006 & 0.356 $\pm$ 0.018 & 0.752 $\pm$ 0.010 & 0.333 $\pm$ 0.025 & 0.810 $\pm$ 0.015 & 0.062 $\pm$ 0.013 \\
DCRNN       & 240.2 & \underline{71.2} & 0.896 $\pm$ 0.003 & 0.513 $\pm$ 0.006 & 0.895 $\pm$ 0.011 & 0.574 $\pm$ 0.019 & 0.877 $\pm$ 0.006 & 0.125 $\pm$ 0.018 \\
GraphS4mer  & 30.29 & 340.4 & 0.884 $\pm$ 0.002 & 0.461 $\pm$ 0.001 & 0.856 $\pm$ 0.017 & 0.474 $\pm$ 0.046 & 0.899 $\pm$ 0.015 & 0.185 $\pm$ 0.046 \\
GRU-GCN     & 31.13 & 4605.5 & 0.894 $\pm$ 0.005 & 0.527 $\pm$ 0.059 & \underline{0.907 $\pm$ 0.001} & \textbf{0.640 $\pm$ 0.012} & 0.907 $\pm$ 0.005 & 0.171 $\pm$ 0.025 \\
EvoBrain    & 57.99 & 3981.7 & \underline{0.897 $\pm$ 0.005} & \textbf{0.558 $\pm$ 0.020} & 0.888 $\pm$ 0.018 & 0.583 $\pm$ 0.038 & \underline{0.917 $\pm$ 0.005} & \underline{0.232 $\pm$ 0.041} \\
\midrule
Ours ($E_h{=}1$) & \textbf{29.10} & 332.2 & \textbf{0.899 $\pm$ 0.006} & \underline{0.551 $\pm$ 0.008} & 0.885 $\pm$ 0.010 & 0.564 $\pm$ 0.041 & \underline{0.917 $\pm$ 0.005} & 0.182 $\pm$ 0.033 \\
Ours ($E_h{=}2$) & \underline{29.70} & 332.5 & 0.894 $\pm$ 0.005 & 0.541 $\pm$ 0.025 & 0.892 $\pm$ 0.002 & 0.537 $\pm$ 0.041 & 0.907 $\pm$ 0.009 & 0.164 $\pm$ 0.026 \\
Ours ($E_h{=}3$) & 30.32 & 332.8 & 0.880 $\pm$ 0.006 & 0.540 $\pm$ 0.002 & \textbf{0.908 $\pm$ 0.006} & \underline{0.636 $\pm$ 0.023} & \textbf{0.924 $\pm$ 0.003} & 0.191 $\pm$ 0.009 \\
\bottomrule
\vspace{-2em}
\end{tabular}
\end{table*}

\paragraph{Tasks.}
We evaluate three EEG seizure-analysis tasks: window-based seizure detection, point-wise seizure detection, and preictal seizure prediction.
In window-based detection, each $W$-second clip ($W \in \{12,60\}$) is labeled positive if it overlaps with any annotated seizure interval.
In point-wise detection, each one-second frame is labeled using seizure onset and offset annotations, so models are evaluated at one-second temporal resolution.
In seizure prediction, clips are classified as preictal or interictal. Preictal clips are drawn from the $\Delta_p{=}60$\,s window before seizure onset, and the interictal pool is constructed by discarding all seizure-labeled segments and excluding a $\Delta_b{=}300$\,s buffer around every seizure boundary so that preictal evidence does not contaminate the negative class. 
Following EvoBrain~\cite{kotoge2026evobrain}, we report prediction on TUSZ only, since CHB-MIT's smaller scale and highly patient-specific pediatric preictal signal are insufficient for cross-patient prediction.

\paragraph{Prediction sampling.}
To match the fixed clip grid used by all models, we accept a clip as preictal iff its full span lies within the strict $\Delta_p{=}60$\,s preictal window for 12\,s clips, and within a $4\times$ clip-length search window before the buffer for 60\,s clips.
For tractability and to keep F1 meaningful, we subsample interictal clips per patient to at most $5\times$ that patient's number of preictal clips, using a fixed seed and applying the same procedure to train, dev, and test splits.

\paragraph{Metrics.}
We use AUROC and F1 score as primary metrics, following EvoBrain~\cite{kotoge2026evobrain}.
Quantitative results are averaged over three random seeds.
We report test AUROC at the checkpoint with the highest dev AUROC, and test F1 at the threshold $\tau^{*}$ that maximizes dev F1.

\paragraph{Hyperparameters.}
We select hyperparameters by dev AUROC. For \method, we grid search the learning rate, weight decay, dropout, hidden dimension $d$, and the temporal-attention gate $\beta \in \{0, 1\}$. The search space, the selected per-task values, and the shared optimization settings are listed in Appendix~\ref{app:params}. Unless otherwise specified, baseline hyperparameters follow their original implementations or the EvoBrain protocol.

\paragraph{Training.}
All models follow the EvoBrain~\cite{kotoge2026evobrain} training setting, STFT preprocessing, 1:1-balanced training sampler, and the window-based runs. For point-wise detection, all baselines are adapted with a shared per-timestep readout head so that they emit one logit per second.

\begin{table*}[t]
\centering
\caption{Point-wise (per-second) seizure detection on TUSZ (12\,s, 60\,s) and CHB-MIT (12\,s). All baselines are paper architectures adapted with a per-timestep readout head, trained with dense BCE on per-second labels.}
\vspace{-1em}
\label{tab:pointwise}
\small
\setlength{\tabcolsep}{4pt}
\renewcommand{\arraystretch}{1.05}
\begin{tabular}{lcccccccc}
\toprule
\multirow{2}{*}{Method} & \multirow{2}{*}{Time (s/ep.)} & \multirow{2}{*}{Memory (MB)}
 & \multicolumn{2}{c}{TUSZ (12\,s)} & \multicolumn{2}{c}{TUSZ (60\,s)}
 & \multicolumn{2}{c}{CHB-MIT (12\,s)} \\
\cmidrule(lr){4-5} \cmidrule(lr){6-7} \cmidrule(lr){8-9} 
& & & AUROC & F1 & AUROC & F1 & AUROC & F1 \\
\midrule
Dense-LSTM                  & \textbf{28.5} & \textbf{415} & 0.869 $\pm$ 0.007              & 0.369 $\pm$ 0.031              & 0.863 $\pm$ 0.012 & 0.388 $\pm$ 0.028 & 0.857 $\pm$ 0.010             & 0.069 $\pm$ 0.016 \\
Dense-CNN-LSTM              & \underline{29.5} & 1804 & 0.873 $\pm$ 0.013             & 0.378 $\pm$ 0.044              & 0.887 $\pm$ 0.005 & 0.410 $\pm$ 0.002 & 0.914 $\pm$ 0.010             & 0.213 $\pm$ 0.026 \\
Dense-BIOT                  & 111.5 & 3889 & 0.858 $\pm$ 0.013            & 0.348 $\pm$ 0.027              & 0.834 $\pm$ 0.036 & 0.315 $\pm$ 0.041 & 0.841 $\pm$ 0.007             & 0.108 $\pm$ 0.016 \\
Dense-GRU-GCN               & 30.0 & 5419 & 0.910 $\pm$ 0.003             & 0.517 $\pm$ 0.016              & 0.918 $\pm$ 0.007 & 0.554 $\pm$ 0.046 & 0.859 $\pm$ 0.003             & 0.086 $\pm$ 0.014\\
Dense-DCRNN                 & 937.0 & 886 & 0.910 $\pm$ 0.004             & 0.499 $\pm$ 0.017              & 	0.922 $\pm$ 0.004 & 0.486 $\pm$ 0.030 & 0.886 $\pm$ 0.020             & 0.148 $\pm$ 0.047 \\
Dense-EvoBrain              & 69.0 & 3515 & 0.918 $\pm$ 0.001 & \textbf{0.575 $\pm$ 0.011}     & 0.916 $\pm$ 0.002 & \underline{0.558 $\pm$ 0.003} & 0.887 $\pm$ 0.004             & 0.141 $\pm$ 0.025 \\ \midrule
\textbf{Ours}               & 33.0 & \underline{521} & 0.913 $\pm$ 0.005              & 0.491 $\pm$ 0.029              & 0.918 $\pm$ 0.007 & 0.512 $\pm$ 0.020 & 0.910 $\pm$ 0.003             & 0.228 $\pm$ 0.017 \\
\quad + ($\lambda{=}0.1$)   & 33.5 & \underline{521} & 0.920 $\pm$ 0.006              & 0.488 $\pm$ 0.040              & 0.923 $\pm$ 0.009 & 0.504 $\pm$ 0.048 & \underline{0.914 $\pm$ 0.006}    & 0.245 $\pm$ 0.007 \\
\quad + ($\lambda{=}0.3$)   & 33.0 & \underline{521} & \textbf{0.926 $\pm$ 0.004}     & \underline{0.548 $\pm$ 0.042}  & \textbf{0.938 $\pm$ 0.004} & \textbf{0.575 $\pm$ 0.049} & \textbf{0.929 $\pm$ 0.006} & \underline{0.232 $\pm$ 0.015} \\
\quad + ($\lambda{=}1.0$)   & 32.0 & \underline{521} & \underline{0.924 $\pm$ 0.006}              & 0.499 $\pm$ 0.055              & \underline{0.933 $\pm$ 0.004} & 0.555 $\pm$ 0.029 & 0.909 $\pm$ 0.008 & \textbf{0.247 $\pm$ 0.009} \\
\bottomrule
\end{tabular}
\vspace{-1em}
\end{table*}

\paragraph{Baselines.}
We compare against ten baselines covering the three families reviewed in Section~\ref{sec:related}.
  \begin{enumerate}
    \item \textit{Sequential models.} LSTM~\cite{hochreiter1997long} (2 layers, hidden 64) and CNN--LSTM~\cite{ahmedt2020neural} (2D conv-pool over channel$\times$frequency, 2-layer LSTM).
    \item \textit{EEG foundation models.} BIOT~\cite{yang2023biot}, LaBraM~\cite{jiang2024large}, and EEGPT~\cite{wang2024eegpt}, each finetuned from the released checkpoint with a linear classification head. LaBraM and EEGPT take raw EEG, BIOT uses its native spectrogram patches.
    \item \textit{Dynamic GNNs.} EvolveGCN~\cite{pareja2020evolvegcn}, DCRNN~\cite{li2018dcrnn_traffic} ($k{=}3$ $k$-NN correlation graph), GraphS4mer~\cite{tang2023modeling}, GRU-GCN~\cite{gao2022equivalence}, and EvoBrain~\cite{kotoge2026evobrain}.
  \end{enumerate}
\paragraph{Point-wise baselines.}
Existing seizure detection baselines (EvoBrain, DCRNN, GRU-GCN, LSTM, CNN-LSTM, BIOT) are originally designed for window-level classification, i.e., they emit a single binary prediction per $T$-second clip. To enable a fair comparison on the point-wise detection task, we adapt each baseline by replacing its terminal readout, typically a last-timestep aggregation followed by a fully connected classifier, with a time-shared per-timestep prediction head. Concretely, the backbone is kept intact and produces a hidden tensor of shape $(B, T, N, d)$, after which the same linear projection and node aggregation are applied independently at every timestep, yielding $(B, T, 1)$ logits. We denote these point-wise variants as Dense-EvoBrain, Dense-DCRNN, etc. This minimal modification preserves every learnable component of the original architecture and isolates the contribution of the dense temporal supervision from broader architectural changes.

\paragraph{Environments.}
All experiments are conducted on a workstation equipped with NVIDIA RTX~A6000 GPUs (48GB VRAM); every reported timing is measured on a single A6000. Our implementation is in PyTorch~2.5.1 with CUDA~12.1, Python~3.10, on Ubuntu~20.04 LTS.

\subsection{Seizure Detection and Prediction}
\label{sec:detection}

\begin{table}[t]
\centering
\caption{Seizure prediction (preictal vs.\ interictal binary classification) on TUSZ.}
\vspace{-1em}
\label{tab:prediction}
\small
\setlength{\tabcolsep}{4pt}
\renewcommand{\arraystretch}{1.05}
\resizebox{\columnwidth}{!}{%
\begin{tabular}{lcccc}
\toprule
\multirow{2}{*}{Method}
 & \multicolumn{2}{c}{12\,s}
 & \multicolumn{2}{c}{60\,s} \\
\cmidrule(lr){2-3} \cmidrule(lr){4-5}
 & AUROC & F1 & AUROC & F1 \\
\midrule
LSTM           & 0.488 $\pm$ 0.003 & 0.250 $\pm$ 0.019 & 0.659 $\pm$ 0.082 & 0.295 $\pm$ 0.019 \\
CNN-LSTM       & 0.418 $\pm$ 0.009 & 0.244 $\pm$ 0.003 & 0.641 $\pm$ 0.019 & 0.357 $\pm$ 0.013 \\
BIOT           & 0.497 $\pm$ 0.010 & 0.260 $\pm$ 0.027 & 0.613 $\pm$ 0.066 & 0.343 $\pm$ 0.057 \\
LaBraM         & 0.464 $\pm$ 0.008 & 0.272 $\pm$ 0.010 & 0.509 $\pm$ 0.050 & 0.255 $\pm$ 0.045 \\
EEGPT          & 0.457 $\pm$ 0.038 & 0.250 $\pm$ 0.062 & 0.640 $\pm$ 0.063 & 0.315 $\pm$ 0.056 \\
EvolveGCN      & 0.450 $\pm$ 0.012 & 0.224 $\pm$ 0.021 & 0.697 $\pm$ 0.013 & 0.304 $\pm$ 0.035 \\
DCRNN          & 0.491 $\pm$ 0.039 & 0.260 $\pm$ 0.015 & 0.606 $\pm$ 0.021 & 0.270 $\pm$ 0.028 \\
GraphS4mer     & 0.543 $\pm$ 0.019 & 0.283 $\pm$ 0.026 & 0.641 $\pm$ 0.147 & 0.336 $\pm$ 0.108 \\
GRU-GCN        & 0.457 $\pm$ 0.012 & 0.243 $\pm$ 0.013 & 0.621 $\pm$ 0.019 & 0.308 $\pm$ 0.028 \\
EvoBrain       & 0.476 $\pm$ 0.000 & 0.233 $\pm$ 0.009 & 0.623 $\pm$ 0.039 & 0.325 $\pm$ 0.044 \\
\midrule
\textbf{Ours} ($E_h{=}1$) & \underline{0.519 $\pm$ 0.006} & 0.251 $\pm$ 0.012 & \textbf{0.758 $\pm$ 0.028} & \textbf{0.434 $\pm$ 0.037} \\
\textbf{Ours} ($E_h{=}2$) & \textbf{0.560 $\pm$ 0.031} & \textbf{0.278 $\pm$ 0.011} & \underline{0.751 $\pm$ 0.054} & \underline{0.359 $\pm$ 0.057} \\
\textbf{Ours} ($E_h{=}3$) & 0.521 $\pm$ 0.065 & \underline{0.272 $\pm$ 0.034} & 0.642 $\pm$ 0.094 & 0.321 $\pm$ 0.063 \\
\bottomrule
\end{tabular}%
}\vspace{-1.5em}
\end{table}

\paragraph{Window-based Detection}
\label{sec:window_detection}
Window-based seizure detection is the standard benchmark setting for recent EEG seizure models, where each fixed-length EEG clip is assigned a single seizure/non-seizure label.
As shown in Table~\ref{tab:tusz_main}, \method attains the highest AUROC on all three benchmarks (TUSZ 12\,s, TUSZ 60\,s, and CHB-MIT 12\,s), and its F1 is comparable to the top method across settings (e.g., within 0.01 of EvoBrain at TUSZ 12\,s and GRU-GCN at TUSZ 60\,s). Against these strongest dynamic-GNN baselines, \method matches GRU-GCN's training time (29.1 vs.\ 31.1 s/epoch), runs about $2\times$ faster than EvoBrain (29.1 vs.\ 58.0 s/epoch), and uses roughly an order of magnitude less peak GPU memory than either baseline. Detailed runtime and memory measurements are reported in Table~\ref{tab:efficiency}, and the accuracy--efficiency trade-off is visualized in Figure~\ref{fig:pareto}.

\paragraph{Point-wise Detection}
\label{sec:pointwise_detection}

The point-wise (per-second) setting is the stricter benchmark, since the model must localize seizure activity within a clip rather than emit a single clip-level label.
Without smoothness regularization, \method is already the strongest or second-strongest on AUROC across all three benchmarks (Table~\ref{tab:pointwise}).
Adding the temporal smoothness penalty further improves performance. Specifically, \method achieves the highest AUROC on every benchmark and the highest F1 on CHB-MIT. Dense-EvoBrain remains marginally ahead only on TUSZ 12 s F1.

Three architectural choices explain \method's effectiveness on per-second prediction. First, the per-channel Mamba backbone preserves the time axis end-to-end, so per-second logits are read off a per-time-step representation rather than reconstructed from a clip-level summary at the readout. Second, the spatiotemporal hyperedge block concentrates cross-channel seizure evidence into a small set of group embeddings, so each per-second decision draws on a focused, coordinated context instead of averaging over every channel.Third, temporal attention further refines the per-second logits by enforcing consistency along the time axis.

\paragraph{Seizure Prediction}
\label{sec:seizure_prediction}

Seizure prediction, which aims to forewarn an oncoming seizure from the preictal EEG window for abortive medication or closed-loop stimulation~\cite{mormann2007seizure, kuhlmann2018seizure}, is the hardest of the three tasks. As shown in Table~\ref{tab:prediction}, most baselines remain near chance on TUSZ 12\,s and reach only the low-to-mid AUROC range on TUSZ 60\,s. On TUSZ 60\,s, \method ($E_h{=}1$) extends substantial leads on both AUROC and F1 over every baseline, including dynamic-GNN methods that match it on detection. On TUSZ 12\,s, it still attains the highest AUROC, with F1 comparable to the best baseline.

The preictal signal is weak, distributed across channels, and changes gradually over the 60\,s interval before onset. In pairwise dynamic graphs, this pattern is represented through many edge-level trajectories, which can make the relevant preictal structure difficult to distinguish from interictal variability. A spatiotemporal hyperedge instead pools the full 60\,s window across both channel and time into a single group embedding, so the discriminative signal is collected in one place rather than reassembled from pieces. The 12\,s setting captures only a fraction of this window and serves as a short-context stress test rather than the primary clinical operating window, which is why absolute numbers stay low across all methods even though \method still leads.

\label{sec:exp}

\section{Additional Analyses}

\begin{table}[t]
\centering
\caption{
Hyperedge block ablation on TUSZ window-based detection. $\Delta$ rows are changes versus the best \method configuration per column.
}
\vspace{-1em}
\label{tab:hyperblock_ablation}
\small
\setlength{\tabcolsep}{4pt}
\renewcommand{\arraystretch}{1.05}
\resizebox{\columnwidth}{!}{%
\begin{tabular}{lcccc}
\toprule
\multirow{2}{*}{Method}
 & \multicolumn{2}{c}{TUSZ 12\,s}
 & \multicolumn{2}{c}{TUSZ 60\,s} \\
\cmidrule(lr){2-3} \cmidrule(lr){4-5}
 & AUROC & F1 & AUROC & F1 \\
\midrule
\method ($E_h{=}1$) & \textbf{0.899 $\pm$ 0.006} & \textbf{0.551 $\pm$ 0.008} & 0.885 $\pm$ 0.010 & 0.564 $\pm$ 0.041 \\
\method ($E_h{=}2$) & 0.894 $\pm$ 0.005 & 0.541 $\pm$ 0.025 & 0.892 $\pm$ 0.002 & 0.537 $\pm$ 0.041 \\
\method ($E_h{=}3$) & 0.880 $\pm$ 0.006 & 0.540 $\pm$ 0.002 & \textbf{0.908 $\pm$ 0.006} & \textbf{0.636 $\pm$ 0.023} \\
\midrule
w/o hyperblock      & 0.888 $\pm$ 0.005 & 0.500 $\pm$ 0.028 & 0.866 $\pm$ 0.016 & 0.544 $\pm$ 0.038 \\
$\Delta$ (vs.\ best) & $-1.1$\,pt & $-5.1$\,pt & $-4.2$\,pt & $-9.2$\,pt \\
\midrule
w/o Mamba & 0.900 $\pm$ 0.003 & 0.540 $\pm$ 0.020 & 0.888 $\pm$ 0.021 & 0.609 $\pm$ 0.054 \\
$\Delta$ (vs.\ best) & $+0.1$\,pt & $-1.1$\,pt & $-2.0$\,pt & $-2.7$\,pt \\
\bottomrule
\end{tabular}%
}\vspace{-1.5em}
\end{table}

We analyze \method along four axes.
We first isolate the contribution of the spatiotemporal hyperedge block (Section~\ref{sec:ablation_block}) and quantify its accuracy--efficiency trade-off against baselines (Section~\ref{sec:effective}). We then visualize what the hyperedges learn on a representative seizure clip (Section~\ref{sec:visualization}) and study its sensitivity to the number of hyperedges $E_h$ (Section~\ref{sec:eh_sensitivity}).

\subsection{Ablation Study}
\label{sec:ablation_block}

To isolate the contribution of the spatiotemporal hyperedge aggregation block, we remove Eqs.~\eqref{eq:member}--\eqref{eq:residual} of Section~\ref{sec:hyperedge} (soft membership $\alpha_{i,k}$, hyperedge aggregation $z_k$ and broadcast $\tilde h'_i$, and the post-residual mixing), keeping the rest of the encoder intact. ``w/o Mamba'' replaces the per-channel Mamba encoder with a GRU in the per-task best configuration ($E_h=1$ for 12\,s, $E_h=3$ for 60\,s), keeping all else unchanged.

Table~\ref{tab:hyperblock_ablation} reports the effect on TUSZ window-based detection. Removing the block drops $1.1$\,pt AUROC and $5.1$\,pt F1 on TUSZ 12\,s, and $4.2$\,pt AUROC and $9.2$\,pt F1 on TUSZ 60\,s, against the best \method configuration on each task. The cost is larger on the 60\,s setting, where the longer clip exposes more cross-channel and cross-time structure for the block to model. 
Swapping the backbone costs less than removing the hyperedge block (Table~\ref{tab:hyperblock_ablation}): the hyperedge aggregation, not the backbone, is the critical component.
This supports the efficiency-effectiveness argument in Section~\ref{sec:rankk}: its value scales with the channel-time grid, and the PMA readout alone cannot replace it.

\subsection{Effectiveness of \method}
\label{sec:effective}

Table~\ref{tab:efficiency} and Figure~\ref{fig:pareto} examine the accuracy-efficiency trade-off on TUSZ 60\,s detection, where \method achieves state-of-the-art performance.
Excluding LSTM as a basic recurrent baseline, only GraphS4mer matches \method's efficiency on both axes ($1.00\times$ time, $1.02\times$ memory, relative to \method ($E_h$=3)). Every other baseline is substantially worse than \method on at least one of the two axes.
The largest gaps appear in the edge-stream dynamic-GNN baselines, where the edge-level temporal model is the dominant memory bottleneck. As Section~\ref{sec:rankk} explains, GRU-GCN and EvoBrain must retain one edge activation per pairwise edge at every time step, whereas \method bypasses dynamic graph construction entirely and stores only one membership weight per token-hyperedge pair. Empirically, the measured total GPU memory ratios are $13.84\times$ (GRU-GCN) and $11.96\times$ (EvoBrain). The same bottleneck appears at inference time: \method reduces per-segment latency by up to 95.5\% relative to these edge-stream baselines.
Large pretrained encoders (BIOT, LaBraM, EEGPT) are both memory-intensive ($3.07$--$10.75\times$) and slow to train ($2.84$--$7.02\times$). The lighter dynamic-GNN baselines, EvolveGCN and DCRNN, use less memory than \method ($0.19\times$, $0.21\times$), but pay for it with $2.40\times$ and $7.92\times$ longer training time.
Together, these choices yield state-of-the-art AUROC on TUSZ 60\,s at substantially lower compute and memory cost than the competitive baselines.

\begin{table}[t]
  \centering
  \caption{
  Computational efficiency of seizure detection on TUSZ 60\,s.
    Inference denotes average milliseconds per EEG segment.
    Memory-cost values above 1 indicate higher cost than ours, whereas values below 1 indicate lower cost.
  }
  \label{tab:efficiency}
  \small
  \setlength{\tabcolsep}{4pt}
  \renewcommand{\arraystretch}{1.05}
  \resizebox{\linewidth}{!}{%
  \begin{tabular}{lrrrr}
  \toprule
  Method & Time (s/ep.) & Infer. (ms/seg.) & Memory (MB) & Memory cost \\
  \midrule
  LSTM       & 31.30 & 0.013 & 88.0    & $0.26\times$ \\
  CNN-LSTM   & 29.90 & 0.264 & 1208.1  & $3.63\times$ \\
  BIOT       & 86.07 & 3.264 & 1850.7  & $5.56\times$ \\
  LaBraM     & 212.9 & 9.429 & 3576.7  & $10.75\times$ \\
  EEGPT      & 151.2 & 4.425 & 1022.1  & $3.07\times$ \\
  EvolveGCN  & 72.89 & 0.691 & 62.3    & $0.19\times$ \\
  DCRNN      & 240.2 & 3.869 & 71.2    & $0.21\times$ \\
  GraphS4mer & 30.29 & 0.181 & 340.4   & $1.02\times$ \\
  GRU-GCN    & 31.13 & 2.340 & 4605.5  & $13.84\times$ \\
  EvoBrain   & 57.99 & 4.781 & 3981.7  & $11.96\times$ \\
  \midrule
  \method ($E_h{=}1$) & 29.10 & 0.216 & 332.2 & $1.00\times$ \\
  \method ($E_h{=}2$) & 29.70 & 0.219 & 332.5 & $1.00\times$ \\
  \method ($E_h{=}3$) & 30.32 & 0.218 & 332.8 & $1.00\times$ \\
  \bottomrule
  \end{tabular}
  }
\end{table}

\begin{figure}[!t]
  \centering
  \includegraphics[width=0.9\linewidth]{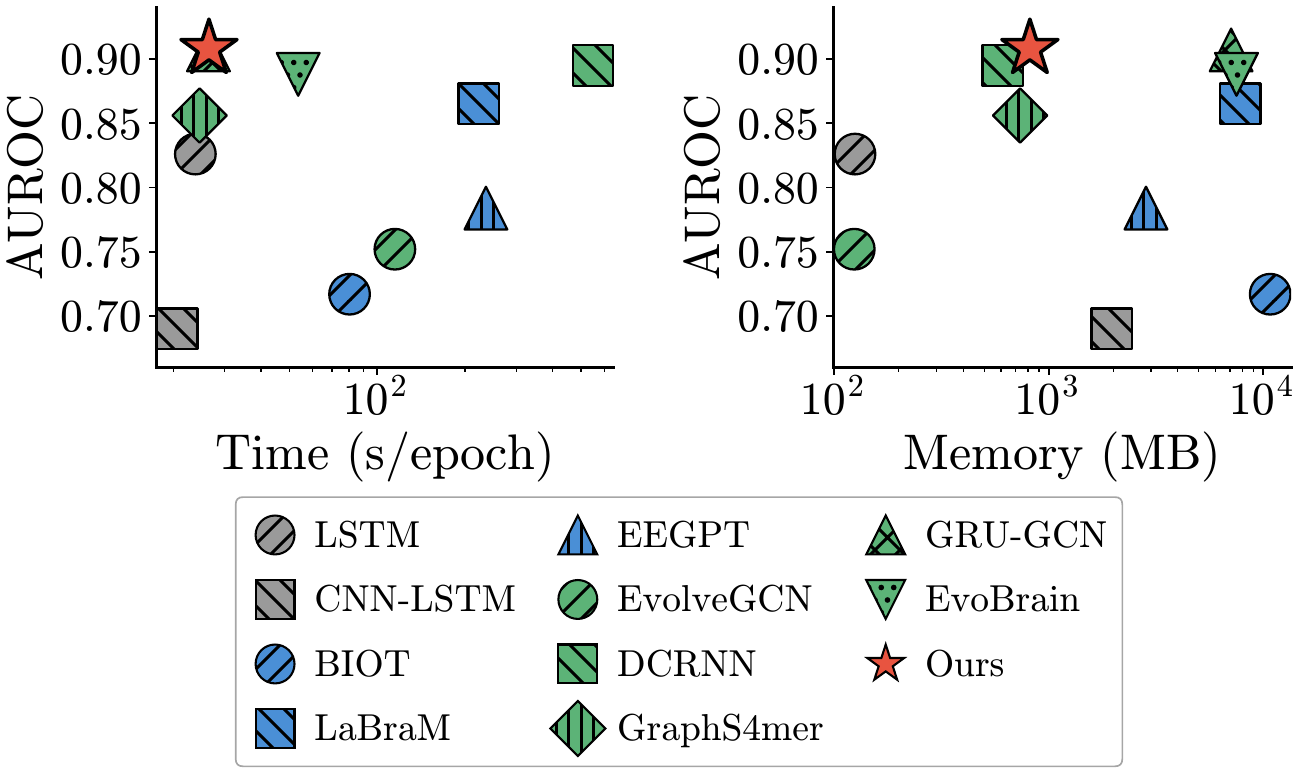}
\caption{Accuracy-efficiency plots on TUSZ 60\,s detection, showing AUROC versus training time per epoch (left) and peak GPU memory (right).}
  \Description{
  Two scatter plots showing AUROC versus time (left) and AUROC versus GPU memory (right) for every baseline and \method. \method occupies the top-left region of both panels.
  }
  \label{fig:pareto}
\end{figure}

\begin{figure*}[!t]
  \centering
  \begin{minipage}[t]{0.62\linewidth}
    \centering
    \includegraphics[width=\linewidth]{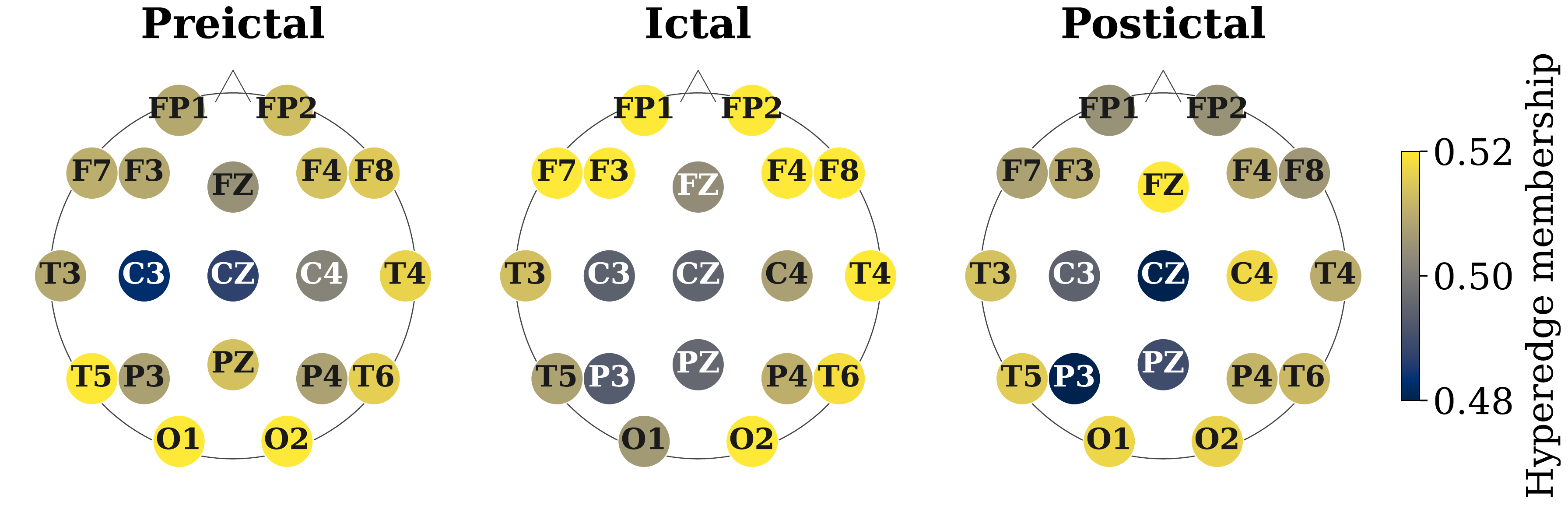}
    \vspace{0.2em}
    {\small (a) Preictal $\to$ ictal $\to$ postictal topomap}
  \end{minipage}
  \hfill
  \begin{minipage}[t]{0.34\linewidth}
    \centering
    \includegraphics[width=\linewidth]{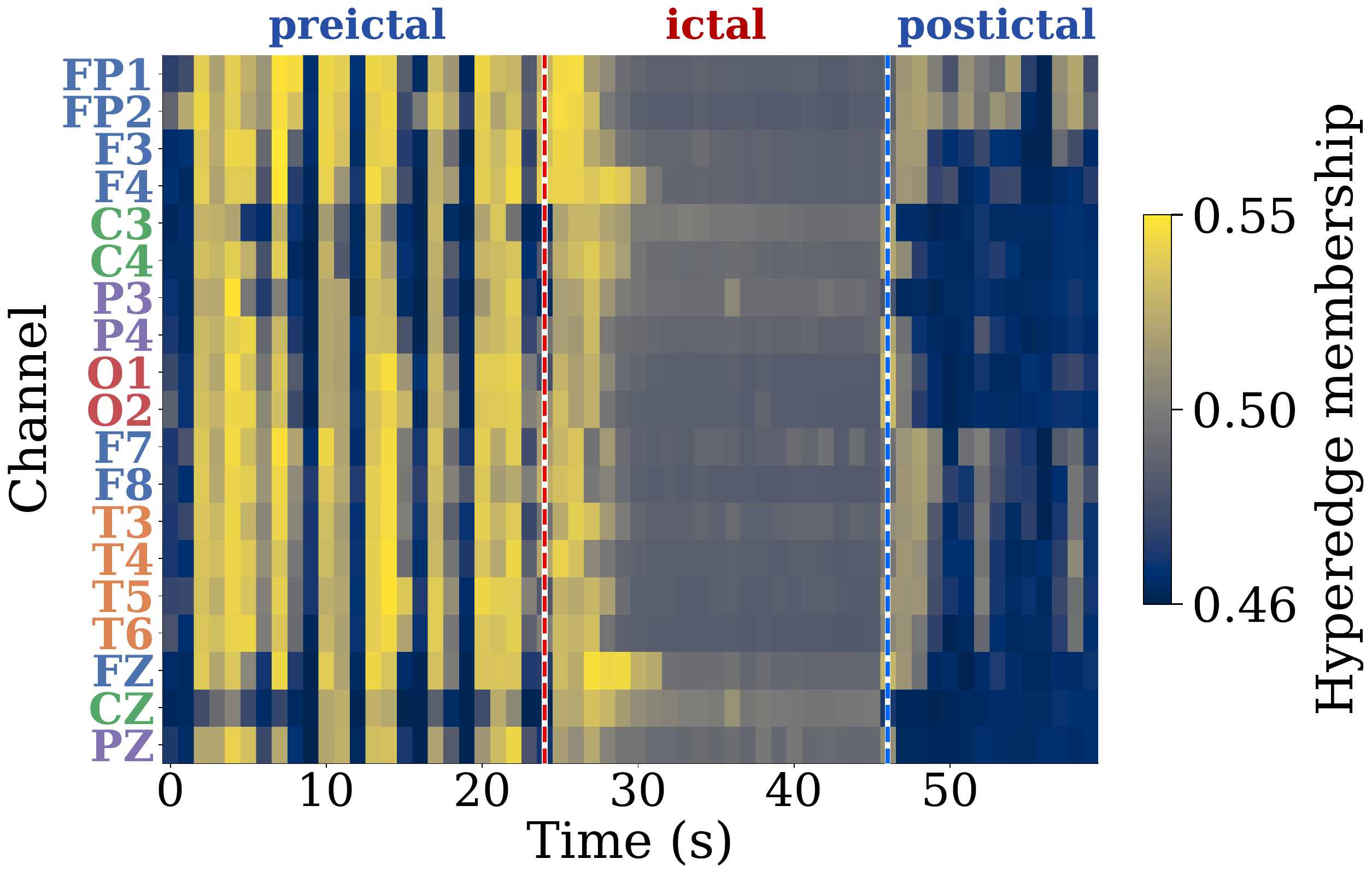}
    \vspace{0.2em}
    {\small (b) Preictal $\to$ ictal $\to$ postictal heatmap}
  \end{minipage}

  \caption{Membership of one learned hyperedge on a TUSZ 60\,s clip that traverses the full preictal $\to$ ictal $\to$ postictal trajectory. (a) Per-channel topomaps for each phase. (b) Membership over time. Dashed lines mark the labeled ictal interval.
  }
  \Description{
  A topomap and a heatmap showing hyperedge membership rising before the labeled seizure onset, peaking during the labeled ictal interval, and decaying after offset.
  }
  \label{fig:viz}
\end{figure*}

\begin{figure}[!t]
\centering
\includegraphics[width=0.95\columnwidth]{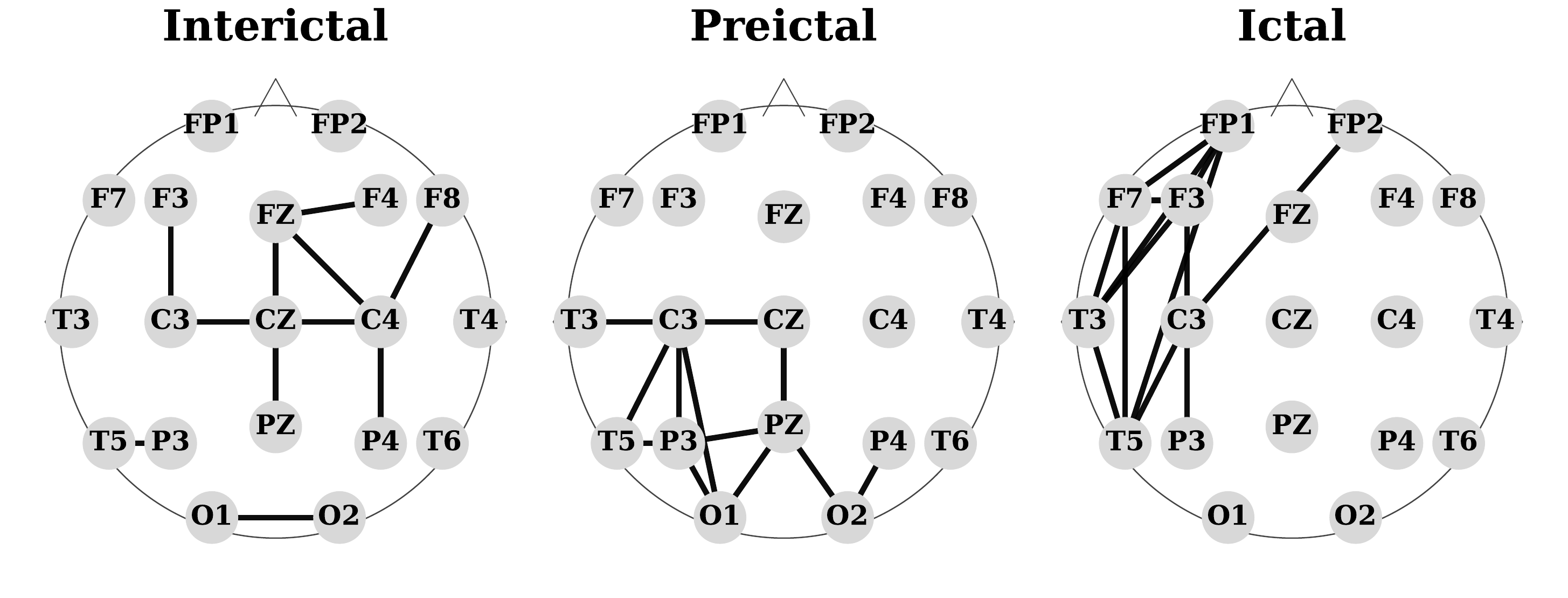}
\caption{
EvoBrain's pairwise EEG connectivity visualization across interictal, preictal, and ictal states, shown alongside Figure~\ref{fig:viz} for comparison.
}
\Description{Three side-by-side topographic plots reproducing EvoBrain's pairwise EEG connectivity visualization for interictal, preictal, and ictal states, drawn as channel-to-channel edges between electrode pairs.}
\vspace{-1.2em}
\label{fig:evobrain_viz}
\end{figure}

\subsection{Visualization: From Pairwise Relations to Spatiotemporal Hyperedges}
\label{sec:visualization}

Figure~\ref{fig:viz} visualizes the membership of one learned hyperedge on a TUSZ 60\,s clip spanning the preictal $\to$ ictal $\to$ postictal trajectory. The membership exhibits fluctuating preictal bursts before the annotated onset, transitions to a sustained and spatially broad pattern during the ictal interval, and is suppressed after the annotated offset. This three-phase membership pattern constitutes a compact, phase-discriminative signature from a single hyperedge.

Figure~\ref{fig:evobrain_viz} provides a pairwise EvoBrain visualization for comparison. EvoBrain represents EEG connectivity through pairwise channel graphs at selected time points, so each visualization shows relations only within a fixed temporal slice. In contrast, \method learns memberships over channel-time tokens, so a single hyperedge can trace a coordinated spatiotemporal pattern across the preictal, ictal, and postictal trajectory. This representation is more compact than a per-time-step pairwise graph and preserves cross-time structure that pairwise visualizations decompose. 

\subsection{Hyperedge Sensitivity}
\label{sec:eh_sensitivity}

\begin{figure}[!t]
  \centering
  \includegraphics[width=0.95\linewidth]{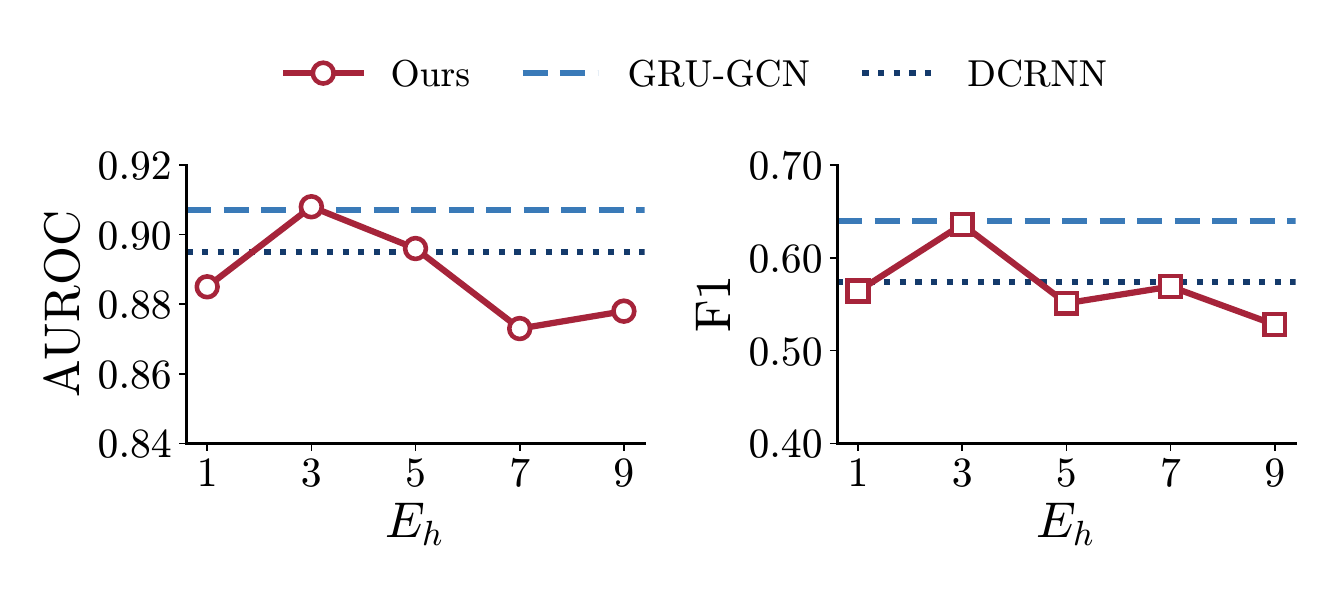}
  \caption{Sensitivity of \method to the number of hyperedges $E_h$ on TUSZ 60\,s detection. Left: AUROC. Right: F1.}
  \Description{
  Two line plots showing AUROC and F1 as the number of hyperedges $E_h$ changes from 1 to 9.
  The curves remain stable across the sweep and stay close to or above the baseline reference lines.
  }
  \label{fig:eh_sensitivity}
  \vspace{-1.2em}
\end{figure}

Figure~\ref{fig:eh_sensitivity} examines how performance varies with the number of hyperedges $E_h$ on the TUSZ 60\,s detection task. Across $E_h \in \{1, 3, 5, 7, 9\}$, AUROC remains within $[0.868, 0.908]$ and F1 within $[0.537, 0.636]$, with maximum spreads of 0.040 and 0.099.

This stability indicates that the seizure-relevant structure on TUSZ is captured by a few coordinated channel-time groups. A single learnable hyperedge ($E_h=1$) already captures the phase-discriminative behavior shown in Figure~\ref{fig:viz}, and $E_h=3$ gives the best 60\,s result within a narrow, stable performance range. This robustness across $E_h$ supports the efficiency argument in Section~\ref{sec:rankk}: \method retains global channel-time reach through a low-rank hyperedge representation, so it can model seizure-related coordination without scaling up to many pairwise edges or many latent groups.
\label{sec:abl}

\section{Conclusion}
We proposed \method, a lightweight encoder that replaces per-time-step pairwise graph construction with a small set of learnable spatiotemporal hyperedges over channel-time tokens. On TUSZ and CHB-MIT, the same encoder reaches state-of-the-art AUROC on window-based detection, one-second point-wise detection, and preictal seizure prediction. The largest margin appears on long-clip prediction, and \method uses up to an order of magnitude less peak GPU memory than the edge-stream dynamic-GNN baselines. 
A single learned hyperedge remains active across the entire preictal $\to$ ictal $\to$ postictal trajectory, which indicates that one group embedding suffices to support both detection and prediction. These results indicate that a compact spatiotemporal hyperedge formulation can match or exceed edge-stream dynamic GNNs on EEG seizure tasks at substantially lower computational cost, particularly when the seizure-relevant signal is distributed across channels rather than concentrated in clearly separable pairwise relations.

\appendix
\section{Additional Hyperedge Visualizations}
\label{app:viz}

Across clips and seizure subtypes, we observe the same qualitative pattern as in Figure~\ref{fig:viz}. Preictal segments show a mild, laterally focused rise in membership, ictal segments become broadly active across channels, and postictal segments retain a decaying residual structure that remains distinguishable from the interictal baseline.

\begin{figure*}[!htbp]
  \centering
  \begin{minipage}[t]{0.45\linewidth}
    \centering
    \includegraphics[width=\linewidth]{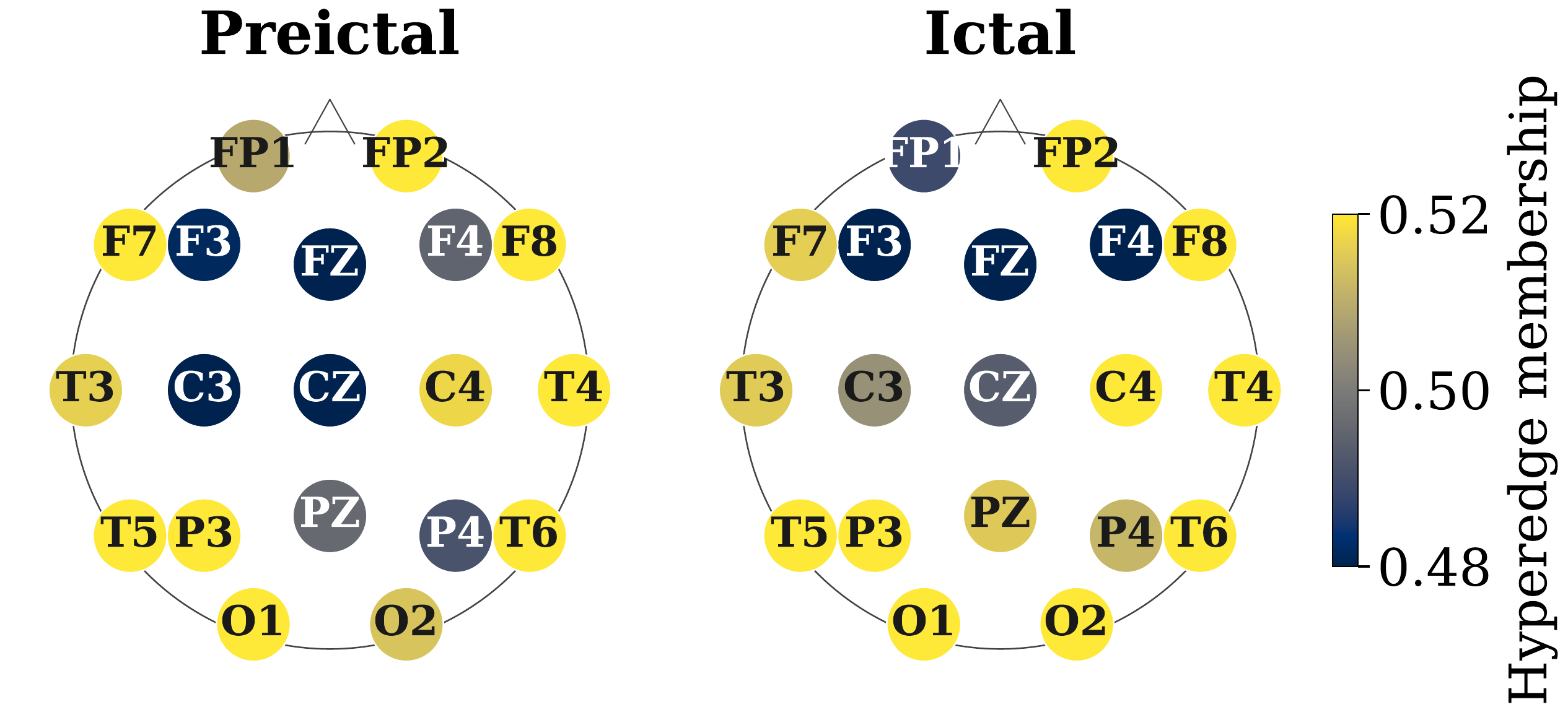}
    \vspace{0.05em}
    {\footnotesize (a) Preictal $\to$ ictal topomap (clip 359)}
  \end{minipage}\hspace{0.02\linewidth}\begin{minipage}[t]{0.3\linewidth}
    \centering
    \includegraphics[width=\linewidth]{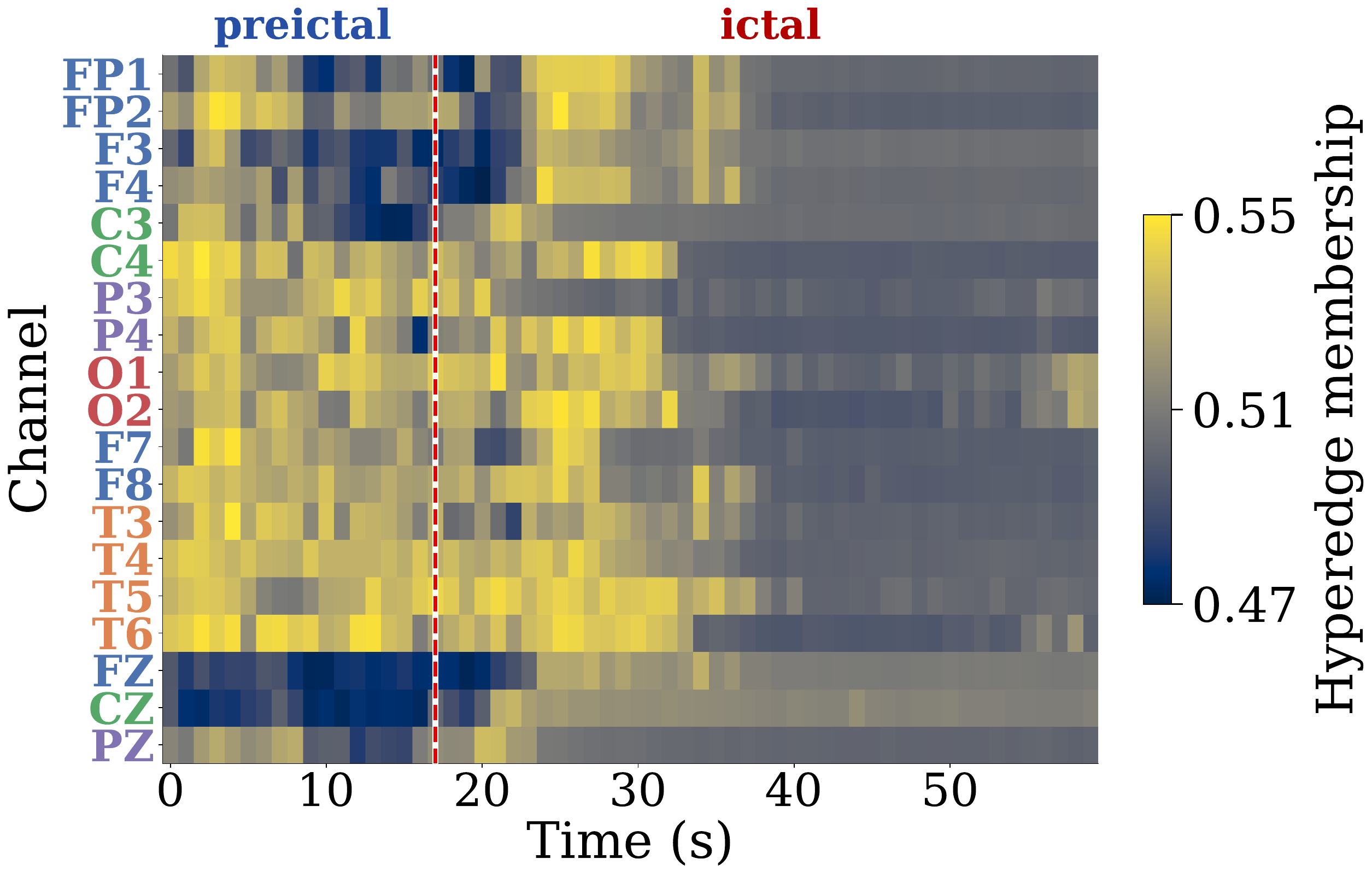}
    \vspace{0.05em}
    {\footnotesize (b) Preictal $\to$ ictal heatmap (clip 359)}
  \end{minipage}

  \vspace{0.1em}

  \begin{minipage}[t]{0.45\linewidth}
    \centering
    \includegraphics[width=\linewidth]{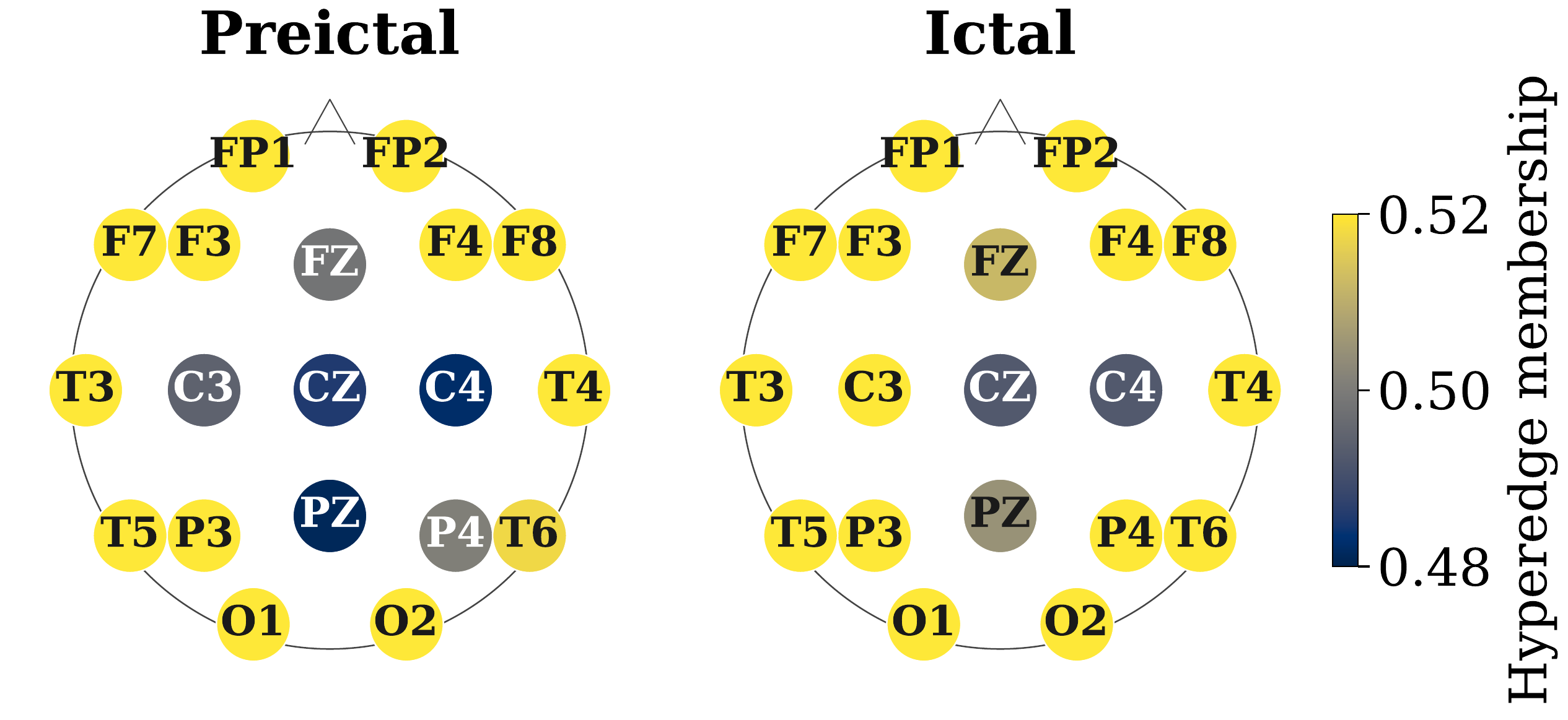}
    \vspace{0.05em}
    {\footnotesize (c) Preictal $\to$ ictal topomap (clip 3907)}
  \end{minipage}\hspace{0.02\linewidth}\begin{minipage}[t]{0.3\linewidth}
    \centering
    \includegraphics[width=\linewidth]{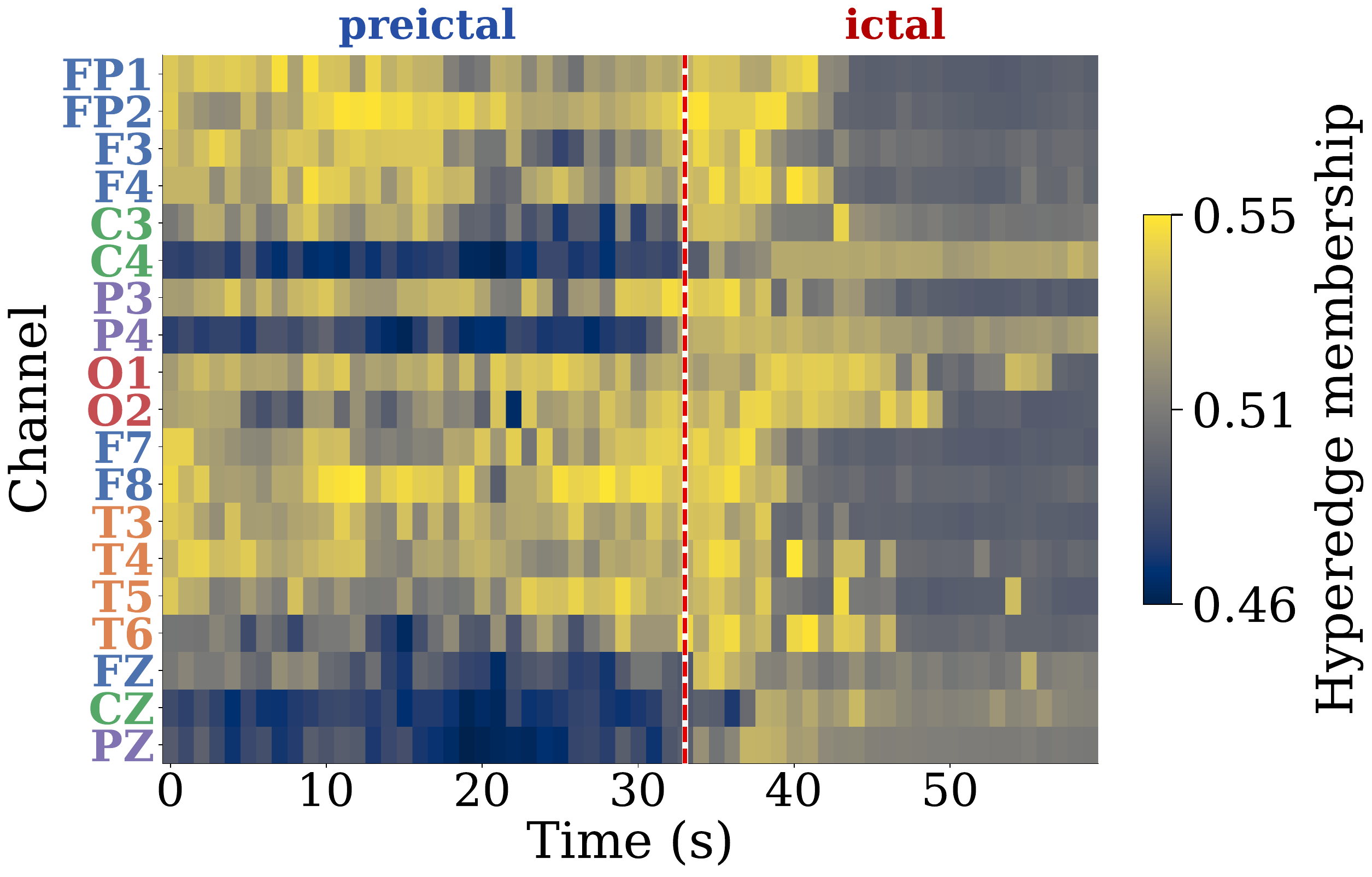}
    \vspace{0.05em}
    {\footnotesize (d) Preictal $\to$ ictal heatmap (clip 3907)}
  \end{minipage}

  \vspace{0.1em}

  \begin{minipage}[t]{0.45\linewidth}
    \centering
    \includegraphics[width=\linewidth]{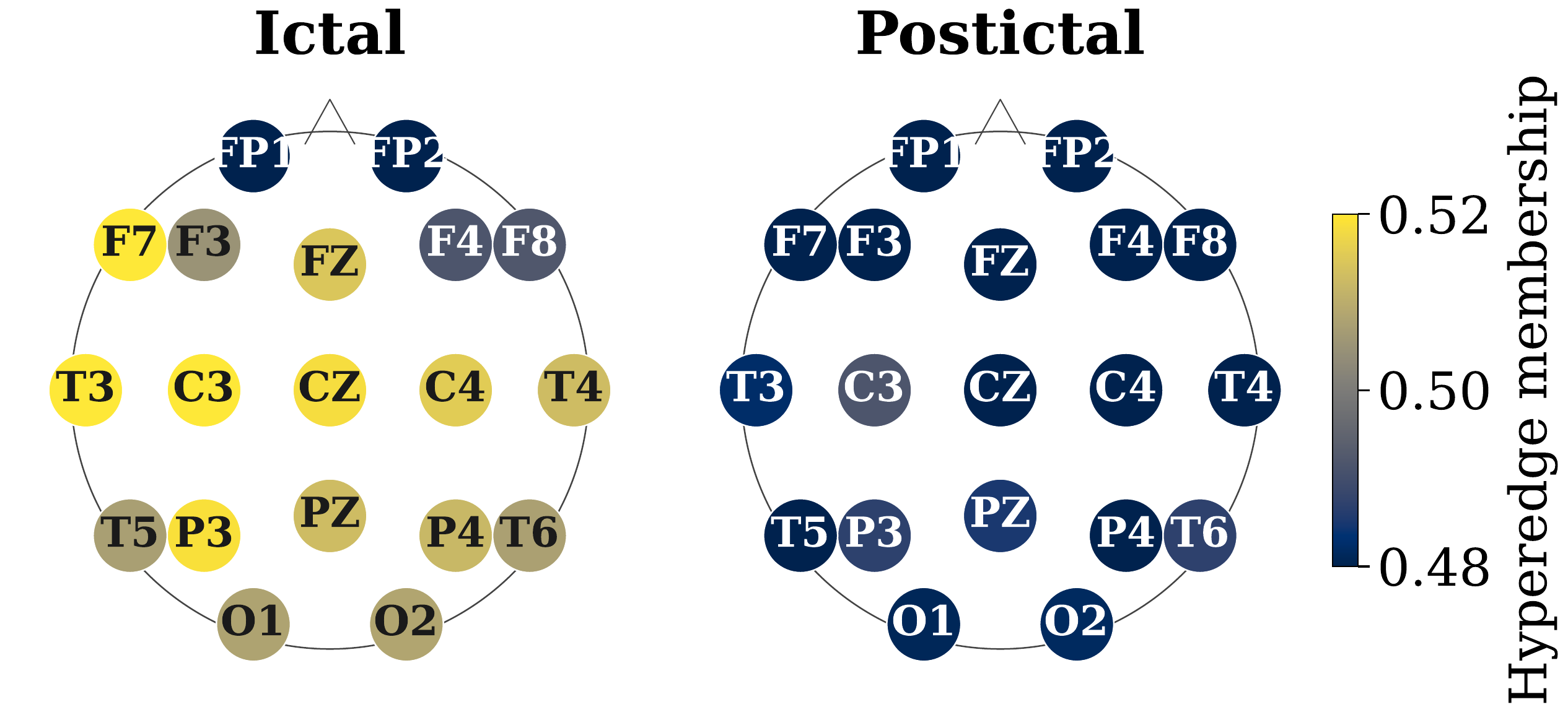}
    \vspace{0.05em}
    {\footnotesize (e) Ictal $\to$ postictal topomap (clip 547)}
  \end{minipage}\hspace{0.02\linewidth}\begin{minipage}[t]{0.32\linewidth}
    \centering
    \includegraphics[width=\linewidth]{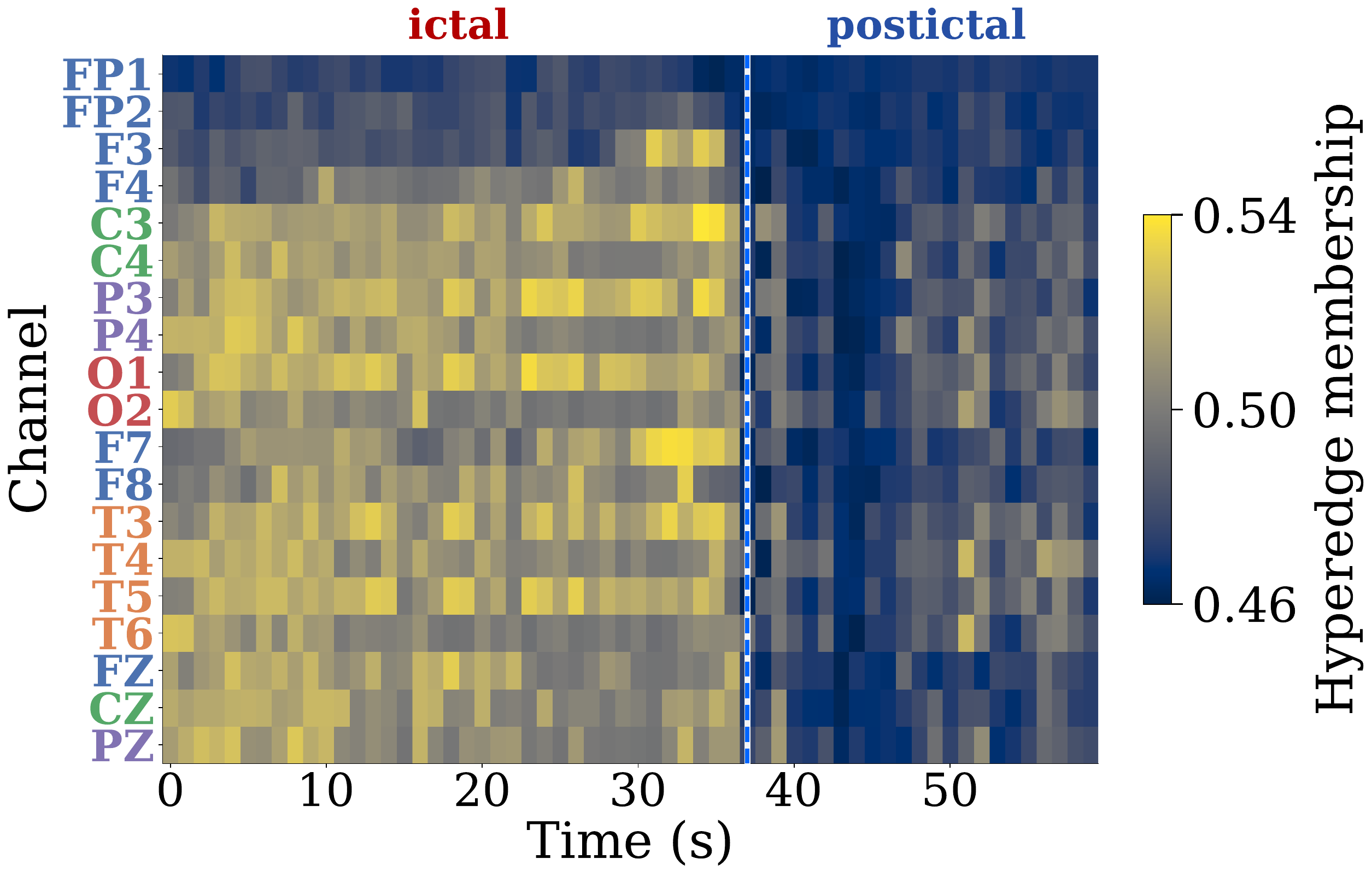}
    \vspace{0.05em}
    {\footnotesize (f) Ictal $\to$ postictal heatmap (clip 547)}
  \end{minipage}

  \vspace{0.1em}

  \begin{minipage}[t]{0.45\linewidth}
    \centering
    \includegraphics[width=\linewidth]{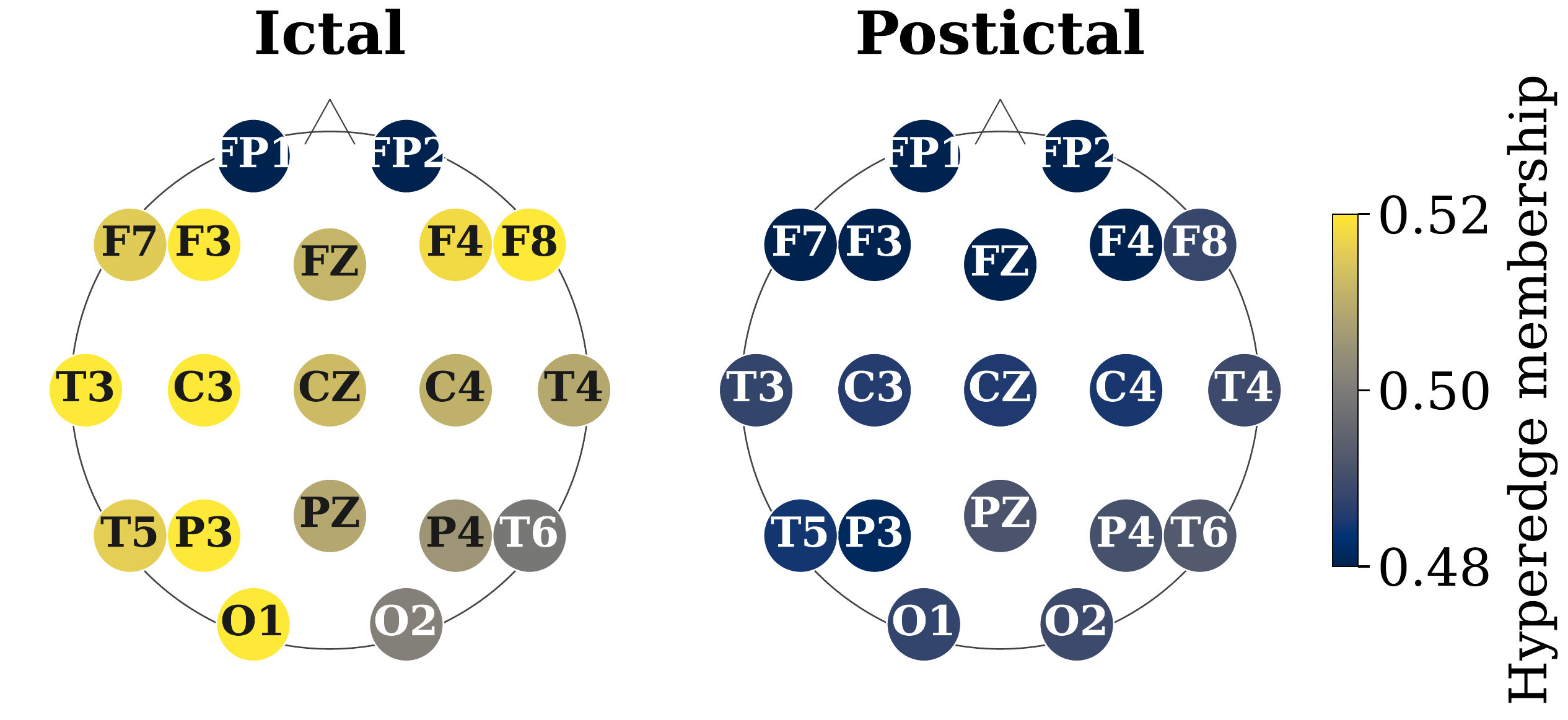}
    \vspace{0.05em}
    {\footnotesize (g) Ictal $\to$ postictal topomap (clip 3937)}
  \end{minipage}\hspace{0.02\linewidth}\begin{minipage}[t]{0.3\linewidth}
    \centering
    \includegraphics[width=\linewidth]{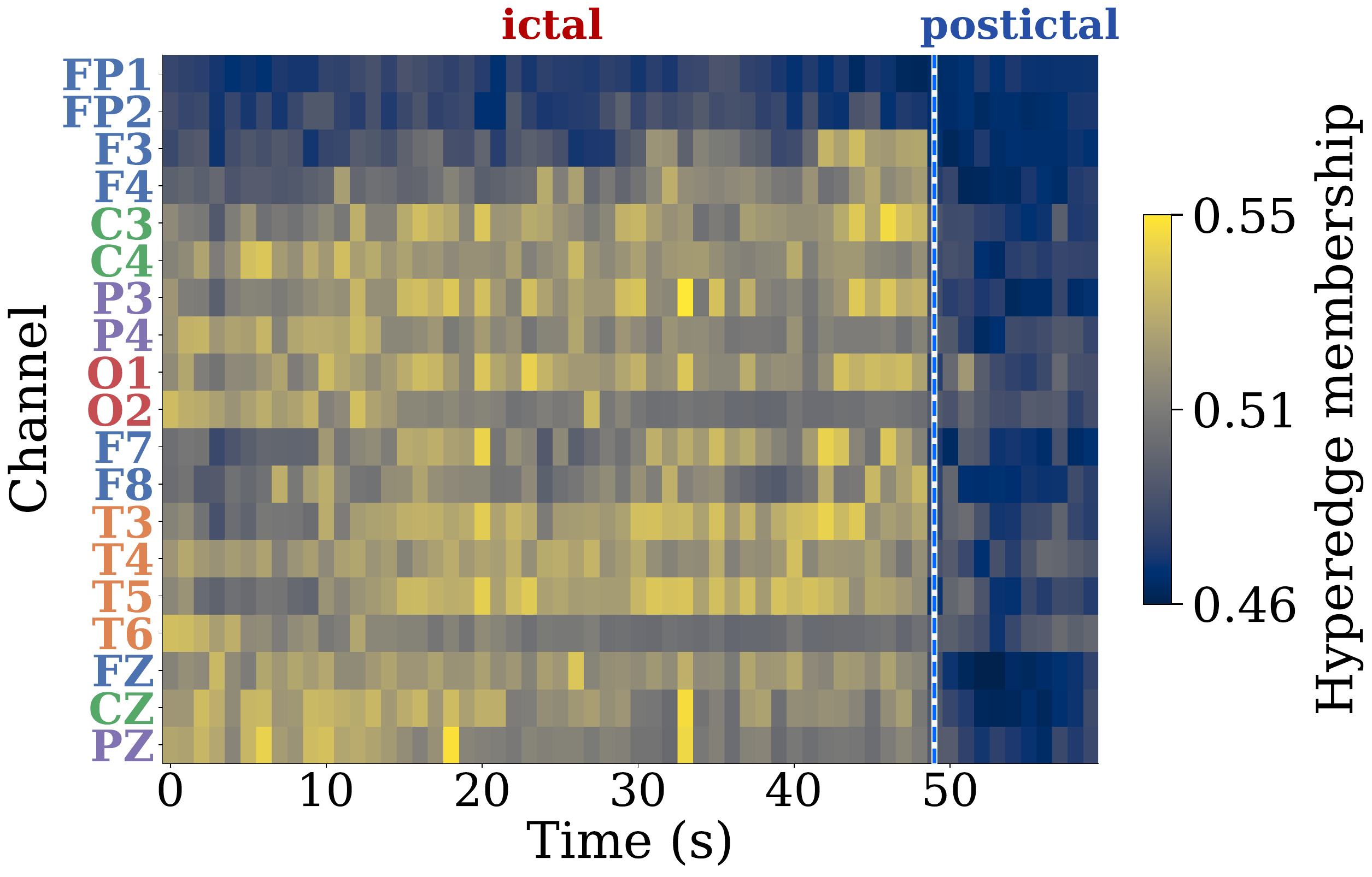}
    \vspace{0.05em}
    {\footnotesize (h) Ictal $\to$ postictal heatmap (clip 3937)}
  \end{minipage}

  \vspace{0.1em}

  \begin{minipage}[t]{0.45\linewidth}
    \centering
    \includegraphics[width=\linewidth]{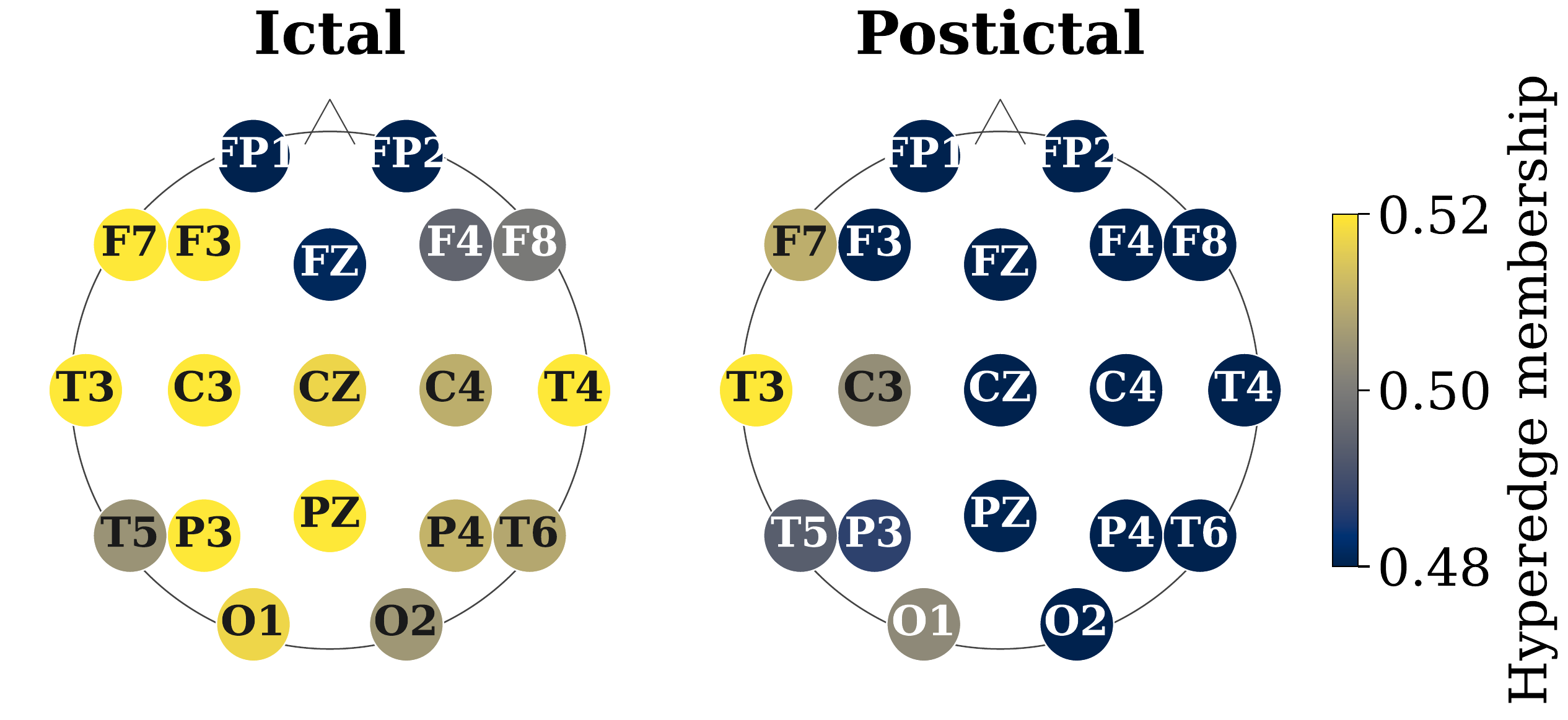}
    \vspace{0.05em}
    {\footnotesize (i) Ictal $\to$ postictal topomap (clip 4303)}
  \end{minipage}\hspace{0.02\linewidth}\begin{minipage}[t]{0.3\linewidth}
    \centering
    \includegraphics[width=\linewidth]{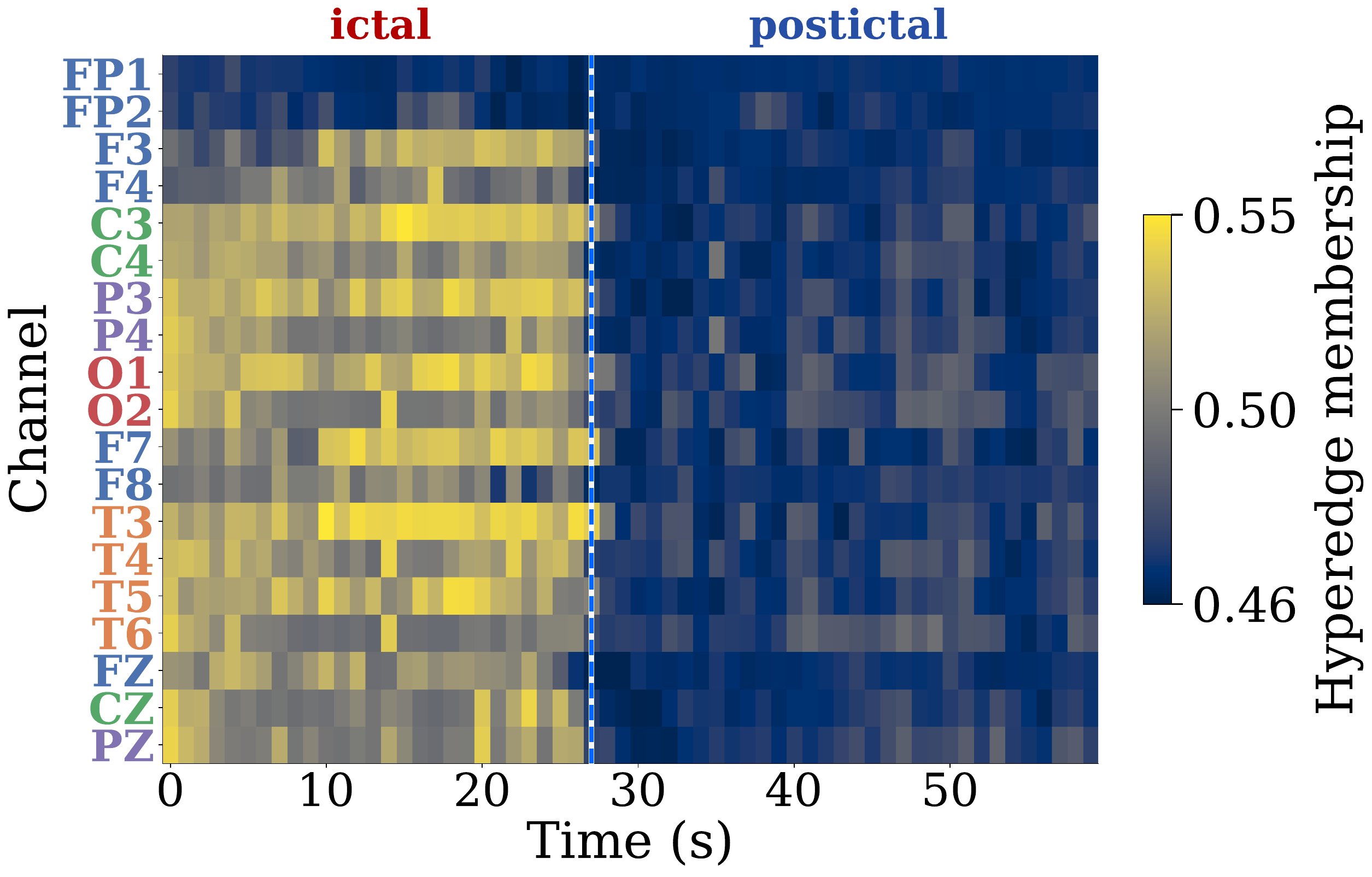}
    \vspace{0.05em}
    {\footnotesize (j) Ictal $\to$ postictal heatmap (clip 4303)}
  \end{minipage}

  \caption{
  Additional learned-hyperedge visualizations on TUSZ 60\,s test clips.
  Rows (a)-(b) and (c)-(d) show preictal $\to$ ictal transitions; rows (e)-(f), (g)-(h), and (i)-(j) show ictal $\to$ postictal transitions.
  In every preictal $\to$ ictal example, the membership rises sharply across channels before or at the annotated onset; in every ictal $\to$ postictal example, the membership decays through the annotated offset.
  }
  \Description{
  Five rows of paired topomap and heatmap visualizations showing the learned hyperedge membership during preictal-to-ictal and ictal-to-postictal transitions across different test clips.
  }
\label{fig:appendix_viz}
\end{figure*}

\section{Hyperparameters}
\label{app:params}
For \method, we grid search learning rate $\{3\times10^{-4}, 5\times10^{-4}, 10^{-3}, 2\times10^{-3}\}$, weight decay $\{5\times10^{-4}, 10^{-3}, 2\times10^{-3}\}$, dropout $\{0.0, 0.1, 0.2\}$, hidden dimension $d \in \{128, 192\}$ (the feature width of the Mamba backbone and hyperedge layers), and the temporal-attention gate $\beta \in \{0, 1\}$, selecting by dev AUROC. The selected hidden dimension is $d=128$ ($d=192$ for 12\,s window detection on both datasets), with $\beta=1$ for detection and $\beta=0$ for prediction. The learning rate is $10^{-3}$ ($5{\times}10^{-4}$ for CHB-MIT 12\,s window detection and $3{\times}10^{-4}$ for 12\,s point-wise detection), and the weight decay is $5{\times}10^{-4}$. Dropout is $0.2$ on TUSZ and $0.1$ on CHB-MIT for window-based detection, and $0.0$ otherwise. All models are trained with Adam and gradient clipping 5 for up to 40 epochs, with early stopping on dev AUROC using patience 5. Training batch sizes are 128 for TUSZ 12\,s window detection, 64 for prediction, and 32 otherwise, and the test batch size is 64.

\bibliographystyle{ACM-Reference-Format}
\balance
\bibliography{bib}

\end{document}